\documentclass[letterpaper,11pt]{article}
\usepackage[margin=1in]{geometry}
\usepackage{amsmath,amsfonts,amsthm,amssymb,bm, verbatim,dsfont,mathtools}
\usepackage{color,graphicx,appendix}
\usepackage{etoolbox}
\usepackage{tikz}
\usepackage{xr,xspace}
\usepackage{todonotes}
\usepackage{paralist}
\usepackage{caption,soul}
\usepackage[ruled,vlined]{algorithm2e}
\usepackage{microtype}%if unwanted, comment out or use option "draft"

\newtheorem{proposition}{Proposition}

\newtheorem{claim}{Claim}

\usepackage{xspace,prettyref}

\newcommand{\diverge}{\to\infty}

\newcommand{\reals}{{\mathbb{R}}}
\newcommand{\integers}{{\mathbb{Z}}}

\newcommand{\expect}[1]{\mathbb{E}\left[ #1 \right]}

\newcommand{\prob}[1]{ \mathbb{P}\left\{ #1 \right\} }

\newcommand{\Bern}{{\rm Bern}}

\newcommand{\iid}{i.i.d.\xspace}
\newrefformat{eq}{(\ref{#1})}
\newrefformat{chap}{Chapter~\ref{#1}}
\newrefformat{sec}{Section~\ref{#1}}
\newrefformat{alg}{Algorithm~\ref{#1}}
\newrefformat{fig}{Fig.~\ref{#1}}
\newrefformat{tab}{Table~\ref{#1}}
\newrefformat{rmk}{Remark~\ref{#1}}
\newrefformat{clm}{Claim~\ref{#1}}
\newrefformat{def}{Definition~\ref{#1}}
\newrefformat{cor}{Corollary~\ref{#1}}
\newrefformat{lmm}{Lemma~\ref{#1}}
\newrefformat{prop}{Proposition~\ref{#1}}
\newrefformat{app}{Appendix~\ref{#1}}
\newrefformat{hyp}{Hypothesis~\ref{#1}}
\newrefformat{thm}{Theorem~\ref{#1}}
\newrefformat{ass}{Assumption~\ref{#1}}

\newcommand{\pth}[1]{\left( #1 \right)}
\newcommand{\qth}[1]{\left[ #1 \right]}
\newcommand{\sth}[1]{\left\{ #1 \right\}}

\newcommand{\abth}[1]{\left | #1 \right |}

\newcommand{\norm}[1]{\left\|{#1} \right\|_2}

\newcommand{\iprod}[2]{\left \langle #1, #2 \right\rangle}

\newcommand{\calS}{{\mathcal{S}}}

\renewcommand{\tilde}{\widetilde}

\theoremstyle{plain}
\newtheorem{theorem}{Theorem}[section]
\newtheorem{lemma}[theorem]{Lemma}

\newtheorem{fact}{Fact}
\theoremstyle{definition}

\theoremstyle{remark}
\newtheorem{remark}[theorem]{Remark}

\newcommand{\fr}[2]{\mbox{$\frac{#1}{#2}$}}

\newcommand{\ignore}[1]{}

\begin{document}
\title{Collaboratively Learning the Best Option, Using Bounded Memory}
\author{Lili Su, Martin Zubeldia, Nancy Lynch\\
~\\
%Department of Electrical Engineering and Computer Science\\
Massachusetts Institute of Technology\\
~\\
{\bf Contact author and e-mail}: Lili Su (lilisu@mit.edu)
}
\date{}

\maketitle

\begin{abstract}
We consider multi-armed bandit problems in social groups wherein each individual has bounded memory %under the Cover and Hellman's notion,
and shares the common goal of learning the best arm/option. % with the highest average reward. % as $t\diverge$.
We say an individual learns the best option if eventually (as $t\diverge$) it pulls only the arm with the highest average reward.
While this goal is provably impossible for an isolated individual, % without social interaction,
we show that, in social groups, this goal can be achieved easily with the aid of social persuasion, i.e., communication.
Specifically, we study the learning dynamics wherein an individual sequentially decides on which arm to pull next based on not only its private reward feedback but also the suggestions provided by randomly chosen peers.
%Specifically, each individual makes its local sequential decisions on which option to pull next  based on both its private reward feedback
% and the suggestions provided by randomly chosen peers.
Our learning dynamics are hard to analyze via explicit probabilistic calculations due to the stochastic dependency induced by social interaction. Instead, we employ the {\em mean-field approximation} method from statistical physics and we show:
%in both a finite T time horizon and the $T \to inifty$ limit.
%We show the learning dynamics have the following quantitative behavior:
\begin{itemize}
	\item With probability $\to 1$ as the social group size $N \diverge $, every individual in the social group learns the best option. % as time $t\diverge$.
	\item Over an arbitrary finite time horizon $[0, T]$, with high probability (in $N$), the fraction of individuals that prefer the best option grows to 1 exponentially fast as $t$ increases ($t\in [0, T]$).
\end{itemize}
A major innovation of our mean-filed analysis is a simple yet powerful technique to deal with absorbing states in the interchange of limits $N \to \infty$ and $t \diverge $. %In particular, we use the second-order imbedded chains of the original continuous-time Markov chain together with a coupling argument to conclude learnability with high probability. %Notably, this technique enables us to deal with the existence of absorbing states.
%
%The transient behavior over finite $[0, T]$ is harder to analyze directly; for this, we use the {\em mean-field approximation} method. %, as is often used in queueing \cite{xu2013supermarket,gamarnik2017delay,tsitsiklis2012power}.
%In particular, we prove that, over any finite time horizon $[0, T]$, the sample paths of the {\em properly scaled} discrete-state Markov chains, as $N \diverge$, concentrate around the {\em unique} solution of an ODE system, in which the fraction of population that pull the best arm/option converges to one exponentially fast.

The {\em mean-field approximation} method allows us to approximate the probabilistic sample paths of our learning dynamics by a deterministic and smooth trajectory that corresponds to the unique solution of a well-behaved system of ordinary differential equations (ODEs). Such an approximation is desired because the analysis of a system of ODEs is relatively easier than that of the original stochastic system. Indeed, in a great variety of fields, differential equations are used directly to model the macroscopic level system dynamics that are arguably caused by the microscopic level individuals interactions in the system. In this work, we rigorously justify their connection. Our result is a complete analysis of the learning dynamics that avoids the need for complex probabilistic calculations; more precisely,
those are encapsulated within the approximation result. The mean-field approximation method might be useful for other stochastic distributed algorithms that might arise in settings where $N$
is sufficiently large, such as insect colonies, systems of wireless devices, and population protocols.

%Our result is a complete analysis of the learning dynamics that avoids the need for complex probabilistic calculations; more precisely, those are encapsulated within the approximation result. The mean-field approximation method might be useful for other stochastic distributed algorithms that might arise in settings where $N$ is sufficiently large such as insect colonies.

\end{abstract}

%{\bf This paper is a regular submission. If this paper is not accepted as a regular paper, please consider it as a brief announcement.}

\newpage

\section{Introduction}
\label{sec: intro}
Individuals often need to make a sequence of decisions among a fixed finite set of options (alternatives), whose rewards/payoffs can be regarded as stochastic, for example:
\begin{itemize}
\item Human society: In many economic situations, individuals need to make a sequence of decisions among multiple options, such as when purchasing perishable products \cite{10.2307/2171959} and when designing financial portfolios \cite{shen2015portfolio}. In the former case, the options can be the product of the same kind from different sellers. In the latter, the options are different possible portfolios.
\item Social insect colonies and swarm robotics: Foraging and house-hunting are two fundamental problems in social insect colonies, and both of them have inspired counterpart algorithms in swarm robotics \cite{pini2012multi}. During foraging, each ant/bee repeatedly refines its foraging areas to improve harvesting efficiency. House-hunting refers to the collective decision process in which the entire social group collectively identifies a high-quality site to immigrate to. For the success of house-hunting, individuals repeatedly scout and evaluate multiple candidate sites, and exchange information with each other to reach a collective decision.
\end{itemize}
Many of these sequential decision problems can be cast as {\em multi-armed bandit problems} \cite{lai1985asymptotically,auer2002finite,bubeck2012regret}. These have been studied intensively in the centralized setting, where there is only one player in the system, under different notions of performance metrics such as pseudo-regret, expected regret, simple regret, etc. \cite{lai1985asymptotically,auer2002finite,bubeck2012regret,mannor2004sample,robbins1956sequential,bubeck2012regret}. Specifically, a $K$-armed bandit problem is defined by the reward processes of individual arms/options $\pth{R_{k, k_i}: k_i \in \integers_+} $ for $k=1, \cdots, K$, where $R_{k, k_i}$ is the reward of the $i$--th pull of arm $k$.
At each stage, a player chooses one arm to pull and obtains some observable payoff/reward generated by the chosen arm. In the most basic formulation the reward process $\pth{R_{k, k_i}: k_i\in \integers_+}$ of each option is stochastic and successive pulls of arm $k$ yield $\iid$ rewards $R_{k, 1}, R_{k, 2}, \cdots$. % which are $\iid$ according to an unknown distribution.
Both asymptotically optimal algorithms and efficient finite-time order optimal algorithms %(with additional restrictions on the reward distributions)
have been proposed \cite{robbins1956sequential,auer2002finite,bubeck2012regret}. These algorithms typically have some non-trivial requirements on individuals' memorization capabilities. %, such as unbounded memory.
For example, upper confidence bound (UCB) algorithm requires an individual to memorize the cumulative rewards of each arm he has obtained so far, the number of pulls of each arm, and the total number of pulls \cite{robbins1956sequential,auer2002finite}. Although this is not a memory-demanding requirement,
%as only cumulative rewards need to be recorded,
nevertheless, this requirement cannot be perfectly fulfilled even by humans, let alone by social insects, due to bounded rationality of humans, and limited memory and inaccurate computation of social insects. In human society, when a customer is making a purchase decision of perishable products, he may recall only the brand of product that he is satisfied with in his most recent purchase. Similarly, in ant colonies, during house-hunting, an ant can memorize only a few recently visited sites.

In this paper, we capture the above memory constraints by assuming an individual has only bounded/finite memory. The problem of multi-armed bandits with {\em finite memory constraint} has been proposed by Robbins \cite{robbins1956sequential} and attracted some research attention \cite{smith1965robbins,cover1968note,1054427}. % in the special {\em Two-Armed Bandit problems} setting. % and the rewards are Bernoulli.
The subtleties and pitfalls in making a good definition of memory were not identified until Cover's work \cite{1054427,cover1968note}. We use the memory assumptions specified in \cite{1054427}, which require that an individual's history  be summarized by a finite-valued memory. The detailed description of this notion of memory can be found in Section \ref{sec: model}. We say an individual learns the best option if eventually (as $t\diverge$) it pulls only the arm with the highest average reward.

%Restricting to two-armed bandits,
%the impact of finite memory has been considered \cite{1054427} in the Bayesian setting, where the distributions of the two arms are known but the identities of the arms are not.
For an isolated individual, learning the best option is provably impossible \cite{1054427}.\footnote{A less restricted memory constraint -- stochastic fading memory -- is considered in \cite{xureinforcement}, wherein similar negative results when memory decays fast are obtained.}
%
%the asymptotic proportion of pulls of the best arm is bounded away from one. % for all admissible learning rules.
%In addition, optimal learning rules, in general, do not exist.
Nevertheless, successful learning is still often observed in social groups such as human society \cite{10.2307/2171959}, social insect colonies \cite{nakayama2017nash} and swarm robotics \cite{pini2012multi}. This may be because in social groups individuals inevitably interact with others. In particular, in social groups individuals are able to, and tend to, take advantage of others' experience through observing others \cite{bandura1969social,rendell2010copy}. Intuitively, it appears that as a result of this social interaction, the memory of each individual is ``amplified'', and this {\em amplified shared memory} is sufficient for the entire social group to collaboratively learn the best option.

\paragraph{Contributions}
In this paper, we rigorously show that the above intuition is correct. We study the learning dynamics wherein an individual makes its local sequential decisions on which arm to pull next based on not only its private reward feedback but also the suggestions provided by randomly chosen peers. Concretely, we assume time is continuous and each individual has an independent Poisson clock with common parameter. The  Poisson clocks model is very natural and has been widely used \cite{ross2014introduction,shwartz1995large,kurtz1981approximation,hajek2015random}: Many natural and engineered systems such as human society, social insect colonies and swarm of robots are not fully synchronized, and not all individuals take actions in a fixed time window; nevertheless, there is still some common pattern governing the action timing, and this common pattern can be easily captured by Poisson clocks. When an individual's local clock ticks, it attempts to perform an update immediately via two steps:
\begin{enumerate}
	\item {\bf Sampling}:
	 {\bf If} the individual does not have any preference over the $K$ arms yet, {\bf then}
	\begin{enumerate}
		\item with probability $\mu\in (0,1]$, the individual pulls one of the $K$ arms uniformly at random (uniform sampling);
		\item with probability $1-\mu$, the individual chooses one peer uniformly at random, and pulls the arm preferred by the chosen peer (social sampling);
	\end{enumerate}
	{\bf else}  the individual chooses one peer uniformly at random, and pulls the arm preferred by the chosen peer (social sampling).
	\item {\bf Adopting}: {\bf If} the stochastic reward generated by the pulled arm is 1, {\bf then} the individual updates its preference to this arm.
\end{enumerate}
Formal description can be found in Section \ref{sec: model}. Our learning dynamics are similar to those studied in \cite{Celis:2017:DLD:3087801.3087820} with two key differences: We relax their synchronization assumption, and we require only individuals without preferences do uniform sampling. These differences are fundamental and require completely new analysis, see Section \ref{subsec: comparison} for the detailed discussion. % As a result of these, our analysis is completely different from that in \cite{Celis:2017:DLD:3087801.3087820}.

The above learning dynamics are hard to analyze via explicit probabilistic calculations due to the stochastic dependency induced by social interaction. Instead, we employ the {\em mean-field approximation} method from statistical physics \cite{stanley1971phase,kurtz1981approximation} to characterize the learning dynamics. To the best of our knowledge, we are the first to use the mean-field analysis for the problem multi-armed bandit in social groups. % both in the limit as $t\diverge$ and over an arbitrary finite time horizon.
\begin{itemize}
\item We show that, with probability $\to 1$ as the social group size $N \diverge $, every individual in the social group learns the best option with local memory of size $ (K+1)$. % -- recalling that $K$ is the number of options.
Note that the  memory size $K+1$ is near optimal, as an individual needs $K$ memory states to distinguish the $K$ arms.

Our proof explores the space-time structure of a Markov chain: We use the second-order space-time structure of the original continuous-time Markov chain; the obtained jump process is a random walk with nice transition properties which allow us to couple this embedded random walk with a standard biased random walk to conclude learnability.
This proof technique might be of independent interest since it enables us to deal with absorbing states of a Markov chain in interchanging the limits of $N\diverge $ and $t \diverge $.

\item Note that the {\em learnability} under discussion is a time-asymptotic notion -- recalling that we say an individual learns the best option if, as $t\diverge$, it pulls only the arm with the highest average reward. In addition to {\em learnability}, it is also important to characterize the transient behavior of the learning dynamics, i.e., at a given time $t$, how many individuals prefer the best arm/option, the second best arm, etc. The transient behavior over finite $[0, T ]$ is harder to analyze directly; for this, we get an indirect characterization. In particular, we prove that, over an arbitrary finite time horizon $[0, T]$, the probabilistic sample paths of the {\em properly scaled} discrete-state Markov chains, as $N\diverge$,  concentrate around a deterministic and smooth trajectory that corresponds to the {\em unique} solution of a system of ordinary differential equations (ODEs).  We further show that in this deterministic and smooth trajectory, the fraction of individuals that prefer the best option grows to 1 exponentially fast as $t$ increases. Therefore, using this indirect characterization, we conclude that over an arbitrary finite time horizon $[0, T]$, with high probability (in $N$),  the fraction of individuals that prefer the best option grows to 1 exponentially fast as $t$ increases $(t \in [0, T ])$.
\end{itemize}
%
%The {\em mean-field approximation} method is very powerful: % and has been widely used in statistical physics and queueing theory \cite{stanley1971phase,kurtz1981approximation,xu2013supermarket,gamarnik2017delay,tsitsiklis2012power}.
%By approximating the probabilistic sample paths of our learning dynamics by a deterministic and smooth trajectory, the analysis has been simplified significantly.
Our result is a complete analysis of the learning dynamics that avoids the need for complex probabilistic calculations; more precisely,
those are encapsulated within the approximation result.
Indeed, in a great variety of fields, differential equations are used directly to model the macroscopic level dynamics that are arguably caused by the microscopic level individuals interactions in the system. In this work, we rigorously justify their connection.
The mean-field approximation method might be useful for other stochastic distributed algorithms that might arise in settings where $N$
is sufficiently large, such as insect colonies, systems of wireless devices, and population protocols.

\section{Model and  Algorithm}
\label{sec: model}
\paragraph{Model}
%\label{subsec: model}
We consider the $K$-armed stochastic bandit problems in social groups, wherein the reward processes of the $K$ arms/options are Bernoulli processes with parameters $p_1, \cdots, p_K$. %An individual chooses one option to pull at a time, and obtains a stochastic reward which is Bernoulli. %distribution associated with the pulled arm.
If arm $a_k$ is pulled at time $t$, then reward $R_t \sim \Bern\pth{p_k}$, i.e.,
\begin{align*}
R_{t} =
\begin{cases}
1, & \text{with probability }p_k; \\
0, & \text{otherwise.}
\end{cases}
\end{align*}
Initially the distribution parameters $p_1, \cdots, p_K$ are unknown to any individual. We assume the arm with the highest parameter $p_k$ is unique. %, i.e., there eixsts $k^*$ such that $p_{k^*} > p_k$ for any $k\not=k^{*}$.
We say an individual learns the best option if, as $t\diverge$, it pulls only the arm with the highest average reward. %\footnote{Note this goal is relatively easy to achieve when an individual has unbounded memory.}
Without loss of generality, let $a_1$ be the unique best arm and  $p_1>p_2\ge \cdots p_K\ge 0$.

A social group consists of $N$ homogeneous individuals. We relax the synchronization assumption adopted in most existing work in biological distributed algorithms \cite{musco2017ant,su2017ant,Celis:2017:DLD:3087801.3087820} to avoid the implementation challenges induced by forcing synchronization. Instead, we consider the less restrictive setting where each individual has an independent Poisson clock with common parameter $\lambda$, and attempts to perform a one-step update immediately when its local clock ticks. The Poisson clocks model is very natural and has been widely used \cite{ross2014introduction,shwartz1995large,kurtz1981approximation,hajek2015random}, see Section \ref{sec: intro} for the detailed discussion.

%because in natural and engineered systems such as human society, social insect colonies and swarm of robots, the actions of individuals are not synchronized -- though there is still some common pattern governing the action timing \cite{ross2014introduction}. %More importantly, forcing synchronization induces additional memory requirement, see the discussion in Section \ref{subsec: comparison} for details.
%\nbr{Concretely, we assume time is continuous and each individual has an independent Poisson clock with common parameter. The  Poisson clocks model is very natural and has been widely used \cite{ross2014introduction,shwartz1995large,kurtz1981approximation,hajek2015random}: Many natural and engineered systems such as human society, social insect colonies and swarm of robots are not fully synchronized, and not all individuals take actions in a fixed time window; nevertheless, there is still some common pattern governing the action timing, and this common pattern can be easily captured by Poisson clocks.}

We assume that each individual has finite/bounded memory \cite{1054427}. %The subtleties and pitfalls in making a good definition of memory was not caught until \cite{cover1968note,1054427}.
%We generalize the notion of memory in \cite{1054427} to social groups, where social interation is available.
%wherein individuals interact among others via social persuasion, i.e., at a time, one individual randomly chooses one peer (including itself) to observe its memory state.
We say an individual has a memory of size $m$ if its experience is completely summarized by an $m$-valued variable $M\in \sth{0, \cdots, m-1}$. As a result of this, an individual sequentially decides on which arm to pull next based on only (i) its memory state and (ii) the information it gets through social interaction. The memory state may be updated with the restriction that only (a) the current memory state, (b) the current choice of arm, and (c) the recently obtained reward, are used for determining the new state.
\paragraph{Learning dynamics: Algorithm}
%
%We consider a simple bio-inspired sequential learning rule, wherein the memory size of an individual is $K+1$.
%Let $\mu \in [0, 1]$.

In our algorithm, each individual keeps two variables:
\begin{itemize}
	\item a local memory variable $M$ that takes values in $\sth{0, 1, \cdots, K}$. If $M=0$, the individual does not have any preference over the $K$ arms; if $M=k \in \{1, \cdots, K\}$, it means that tentatively the individual prefers arm $a_k$ over others.
	\item an arm choice variable $c$ that takes values in $\sth{0, 1, \cdots, K}$ as well. If $c=0$, the individual pulls no arm; if $c=k \in \{1, \cdots, K\}$, the individual chooses arm $a_k$ to pull next.
\end{itemize}
Both $M$ and $c$ are initialized to $0$. When the clock at individual $n$ ticks at time $t$, we say individual $n$ obtains the memory refinement token. With such a token, individual $n$ refines its memory $M$ according to Algorithm \ref{alg: 1} via a two-step procedure inside the most outer {\bf if} clause.  The {\bf if}--{\bf else} clause describes how to choose an arm to pull next: If an individual does not have any preference ( i.e., $M=0$), $c$ is determined through a combination of uniform sampling and social sampling; otherwise, $c$ is completely determined by social sampling, see Section \ref{sec: intro} for the notions of {\em uniform sampling} and {\em social sampling}.
 The second {\bf if} clause says that as long as the reward obtained by pulling the arm is 1, then $M\gets c$; otherwise, $M$ is unchanged.

\begin{algorithm}
\caption{Collaborative Best Option Learning}
\label{alg: 1}
{\bf Input}: $\mu\in (0,1]$, $K$, $N$\;
~~ {\em Local variables}: $M \in \sth{0, 1, \cdots, K}$  and $c\in \sth{0, 1, \cdots, K}$\;
~~ {\em Initialization}: $M=0$, $c=0$ \;

\vskip 0.2\baselineskip

\SetKwFunction{FSO}{SocialObservation}

%%
%%\STATE{$c\gets M$}; \;
\If{local clock ticks}
{
\eIf{$M=0$}
{
	With probability $\mu$, set $c$ to be one of the $K$ arms uniformly at random\;

	%\IF{the obtained reward is one, i.e., $R_t=1$}
	%\STATE{$M \gets c_t$}
	%\ENDIF
    With probability $1-\mu$, $c \gets$ \FSO \;  % \text{\bf SocialObservation}\;
}
{$c \gets$  \FSO \;}

Pull arm $c$\;

\vskip 0.8\baselineskip

\If{$R_t=1$}{$M \gets c$;}

}

\vskip 0.6\baselineskip
\FSO{}
{ \\
Choose one peer $n^{\prime}$ (including itself) uniformly at random\;

\KwRet $M^{\prime}$;     ~~~~~  \%\% $M^{\prime}$ is the memory state of $n^{\prime}$\;

}
\end{algorithm}
%
%The variable $c$ in Algorithm \ref{alg: 1} records the arm to pull -- in order to get reward 1.

\begin{remark}
Note that we assume $\mu\in (0,1]$. When $\mu=0$, the problem is straightforward.  Suppose $\mu=0$, the learning dynamics given by Algorithm \ref{alg: 1} reduces to pure imitation. Since no individuals spontaneously scout out the available options, and no individuals in the social group have any information about the parameters of these options. Thus, the memory state at any individual remains to be $M=0$ throughout the execution, and no individuals can learn the best option. %Henceforth, we assume $\mu\in (0,1]$.
\end{remark}

\paragraph{System state}
\label{sec: markov-chain}
For a given $N$, the learning dynamics under Algorithm \ref{alg: 1} can be represented as a continuous-time random process $ \pth{X^N(t): \, t\in \reals_+ } $ such that
\begin{align}
\label{dl mc}
%\nonumber
  X^N(t) & = \qth{X^N_0(t), X^N_1(t), \cdots, X^N_K(t)}, ~~\text{and} ~ X^N(0) = \qth{N, 0, \cdots, 0} ~~ \in ~~\integers^{K+1},
\end{align}
where $X^N_0(t)$ is the number of individuals whose memory states are 0 at time $t$ with $X^N_0(0)=N$; $X^N_k(t)$, for $k\not=0$, is the number of individuals whose memory states are $k$ at time $t$ with $X^N_k(t)=0$. % for $k\not=0$.
We use $x^N(t) =\qth{x^N_0(t), x^N_1(t), \cdots, x^N_K(t)}$ to denote the realization of $X^N(t)$.  Note that the total population is conserved, i.e., $\sum_{k=0}^K x^N_k(t) =N, \forall ~ t,$
%\begin{align}
%\label{population conservation}
%\sum_{k=0}^K x^N_k(t) =N, ~~~ \forall ~ t,
%\end{align}
for every sample path $x^N(t)$. In fact, a system state is a partition of integer $N$ into $K+1$ non-negative parts, and the state space of $ \pth{X^N(t): \, t\in \reals_+ } $ contains all such partitions. %It is easy to see that for a fixed $N$, a sample path (as a function of $t$) is right-continuous with left limits ($c\grave{a}dl\grave{a}g$), and piece-wise constant except for the countably many jumps.
%From the above learning dynamics, we know that conditioning on $X^N (t)=x^N(t)$, % =\qth{x^N_0(t), x^N_1(t), \cdots, x^N_K(t)}$,$$X^N (t+s)\mid  x^N(t)$$
%is independent of $X^N (\tau)$ for $\tau < t$ and $s>0$.
It is easy to see that for a given $N$, the continuous-time random process $X^N(t)$ defined in \eqref{dl mc} is a Markov chain.

For ease of exposition, we treat the case when $c=0$ as pulling the NULL arm $a_0$, which can generate Bernoulli rewards with parameter $p_0=0$. %That is, with probability zero, an individual obtains positive rewards from pulling option $a_0$.
With this interpretation, the entries in $\pth{X^N(t): t\in \reals^+}$ can be viewed as $K+1$ coupled birth-death processes. The birth rates for the regular arm $a_k$, where $1\le k\le K$, is
\begin{align}
\label{arrival rate 1}
X^N_0(t)\lambda \pth{\frac{\mu}{K} +\pth{1-\mu}\frac{X^N_k(t)}{N}} p_k + \pth{\sum_{ \substack{k^{\prime}: 1\le k^{\prime} \le K\\ k^{\prime}\not=k}} X^N_{k^{\prime}}(t)\lambda }\frac{X^N_k(t)}{N} p_k.
\end{align}
For convenience, we define the birth rate for the NULL arm\footnote{The birth rate for $a_0$ is always 0. As can be seen later, we have the expression written out for ease of exposition.} $a_0$ as
\begin{align}
\label{arrival rate 0}
\pth{N -X^N_0(t)}\lambda \frac{X^N_0(t)}{N} p_0.
\end{align}
We now provide some intuition for the rate in \eqref{arrival rate 1}. Recall that every individual has an independent Poisson clock with rate $\lambda$, and $X^N_0(t)$ is the number of the individuals whose memory $M=0$ at time $t$. So $X^N_0(t)\lambda$ is the rate for such individuals to obtain a memory refinement token. With probability $\frac{\mu}{K} +\pth{1-\mu}\frac{X^N_k(t)}{N}$, such an individual pulls arm $a_k$, which generates reward 1 with $p_k$.
%i.e., $c=k$. Under Algorithm \ref{alg: 1}, an individual updates its memory to $k$ if it obtains $R_t=1$,  which occurs with probability $p_k$.
Similarly, $\sum_{ \substack{k^{\prime}: 1\le k^{\prime} \le K\\ k^{\prime}\not=k}} X^N_{k^{\prime}}(t)\lambda$ is the rate for an individual with $M\not=k$ and $M\not=0$ to obtain a refinement token. With probability $\frac{X^N_k(t)}{N}$, such an individual chooses one peer whose memory state is $k$ through social sampling. Pulling arm $a_k$ generates reward 1 with probability $p_k$. Combining the above two parts together, we obtain the birth rate in \eqref{arrival rate 1}. As per \eqref{arrival rate 0}, the birth rate for the NULL arm is always zero. Nevertheless, the form in \eqref{arrival rate 0} has the following interpretation: With rate $\pth{N -X^N_0(t)}\lambda$, an individual with $M\not=0$ obtains the refinement token; during social sampling, with probability $ \frac{X^N_0(t)}{N}$, a peer individual with memory state $0$ is chosen. % Since NULL option never generates positive reward, the individual with the token never switches to NULL option.
By similar arguments, it is easy to see that the death rate for the NULL arm $a_0$ is
\begin{align}
\label{death rate 0}
X^N_0(t)\lambda \sum_{k=1}^K \pth{\frac{\mu}{K} +(1-\mu)\frac{X^N_k(t)}{N}}p_k,
\end{align}
and for the regular arm $a_k$, where $1\le k\le K$, is
\begin{align}
\label{death rate 1}
X^N_k(t)\lambda \times \sum_{ \substack{k^{\prime}: 1\le k^{\prime} \le K\\ k^{\prime}\not=k}}\frac{X^N_{k^{\prime}}(t) }{N} p_{k^{\prime}}.
\end{align}

Let $s=\qth{s_0, \cdots, s_K}$ be a valid state vector (a proper partition of integer $N$)
%the state vector that satisfies the population conservation condition in \eqref{population conservation},
and $e^k\in \reals^{K+1}$ be the unit vector labelled from the zero--th entry with the $k$--th entry being one and everywhere else being zero. For example, $e^1 = \qth{0, 1, \cdots, 0}$.
The generator matrix $Q^N$ can be expressed as follows:
\begin{align}
\label{state re: generator}
q^N_{s, s+\ell}=
\begin{cases}
%(N-s_0) \lambda \frac{s_0}{N} p_0, ~~\text{if}~ \ell =e^0;\\
s_0\lambda \pth{\frac{\mu}{K} +(1-\mu)\frac{s_k}{N}}p_k, ~~\text{if}~ \ell =e^k ~~ \text{for }k=1, \cdots, K;\\
s_{k^{\prime}}\lambda \frac{s_k}{N} p_k, ~~\text{if}~ \ell = e^{k} - e^{k^{\prime}},\text{for} ~k^{\prime}\not=k ~ \& ~k^{\prime}, k =1, \cdots, K;\\
 -\sum_{k=1}^K s_0\lambda \pth{\frac{\mu}{K} +(1-\mu)\frac{s_k}{N}}p_k - \sum_{k=1}^K \sum_{k^{\prime}: k^{\prime}\not=k} s_{k^{\prime}}\lambda \frac{s_k}{N} p_k, ~~\text{if}~ \ell = {\bf 0}\in \reals^{K+1}; \\
0,  ~~~ \text{otherwise}.
\end{cases}
\end{align}
%The transition probability can be read off from the generator matrix. In particular, let $P^n_{s_1, s_2}(h)$ be the transition probability from state $s_1$ to state $s_2$ during a short time interval $h$. It holds that
%\begin{align*}
%  P^n_{s_1, s_2}(h) &  = I_{\{s_1=s_2\}} + h ~Q^n_{s_1, s_2}+o(h),
%\end{align*}
%where $o(h)$ represents a quantity such that $\lim_{h\to 0}o(h)/h =0$. Note that when $s_1=s_2$, $ P^n_{s_1, s_2}(h)\le 1$. Intuitively, this is because the process has the tendency to move away from the current state.
%
%\begin{remark}
%The above $K+1$ coupled birth-death processes can also be viewed as a closed queuing network with $K+1$ servers and $N$ total customers. The network is closed because there is no external arrivals and no system exits. An individual's Poisson clock ticks and updates memory corresponds to a customer departs from one ``queue'' joins another ``queue'' immediately. The reward parameters of the arms capture different capabilities of the queues in attracting customers. The learning goal is to have all the customers join and stay at the most attractive ``queue''.
%\end{remark}
% 
%\begin{remark}
%Note that one important property of the above opinion dynamics is: once an individual adopts an opinion, it will never switch back to the state in which it does not adopt any opinion. %As a result of this, for any system state $s\in \reals_{+}^{k+1}$ such that $\sum_{i=0}^k s_i=n$, if $s_0>0$, then $s$ is transient.
%\end{remark}
%
%
%
%
\section{Main Results}
\label{sec: main results}
It is easy to see from \eqref{state re: generator} that the Markov chain $\pth{X^N(t): \, t\in \reals_+}$ defined in \eqref{dl mc} has exactly $K$ absorbing states. In particular, each $N\cdot e^k$ (for $k=1, \cdots, K$) is an absorbing state as the rate to move away from $N\cdot e^k$ is zero.
%$q^N_{N \cdot e^k, N \cdot e^k + \ell} =0$ for all $\ell$ -- recalling that $e^k \in \reals^{K+1}$ is the unit vector whose entries are labelled from zero, and its $k$--th entry being one and everywhere else being zero.
%
We are interested in characterizing the probability that the random process $\pth{X^N(t): ~ t\in \reals_+}$ gets into the absorbing state $\qth{0, N, 0, \cdots, 0}$, wherein the memory states of all individuals are $1$.  For notational convenience, let
%\vskip -1.2\baselineskip 
\begin{align}
\label{learning goal}
x^* \triangleq \qth{0, N, 0, \cdots, 0}.
\end{align}

\subsection{Learnability}
From the description of Algorithm \ref{alg: 1}, we know that when the system enters this state, every individual pulls the best arm whenever its local clock ticks, i.e., every agent learns the best option. Define the success event $E^N$ as:
\begin{align}
\label{event: success}
E^N \triangleq \sth{\lim_{t\diverge } X^N(t) = x^*} \subseteq \sth{\text{every individual learns the best option} }.
\end{align}
It turns out that the larger the group size $N$, the more likely every individual in the group learns the best option,
%We show that with probability $\to 1$ as the social group size $N \diverge $, every individual in the social group learns the best option, i.e., $\lim_{N\diverge}\prob{E^N}=1$,
formally stated in the next theorem.
\begin{theorem}
\label{thm: eventual learn}
For any $\delta, \mu, p_1\in (0,1]$ and $p_2\not=0$, we have
%Let $t_c =\frac{1}{\lambda}$. With probability at least
%$ 1- \pth{\frac{p_1}{p_2}}^{- (1-\delta)  \frac{\mu p_1}{K e} N} -  e^{-\frac{\mu p_1}{K e} \frac{\delta^2}{2} N },$ every individual learns the best arm, i.e.,
\begin{align*}
&\prob{E^N}\ge 1- \pth{\frac{p_1}{p_2}}^{-(1-\delta)  \frac{\mu p_1}{K e} N} -  e^{-\frac{\mu p_1}{K e} \frac{\delta^2}{2} N }, ~~~ \text{where} ~ e\approx 2.7183.
\end{align*}
\end{theorem}
As can be seen later, the parameter $\delta$ is introduced for a technical reason.
Theorem \ref{thm: eventual learn} says that the probability of ``every individual learns the best option'' grows to 1 exponentially fast as the group size $N$ increases. Theorem \ref{thm: eventual learn} also characterizes how does $\prob{E^N}$ relate to
(i) $p_1$ -- the performance of the best arm,
(ii) $\frac{p_1}{p_2}$ -- the performance gap between the best arm and the second best arm,
(iii) $\mu$ -- the uniform sampling rate of the individuals temporarily without any preference, and
(iv) $K$ -- the number of arms.
In particular, the larger $p_1$, $\frac{p_1}{p_2}$, and $\mu$, the easier to learn the best option; the larger $K$ (the more arms), the harder to learn the best option. %\nb{LS: Can we simply set $\mu=1$???}

\begin{proof}[{\bf Proof Sketch of Theorem \ref{thm: eventual learn}}]
The main idea in proving Theorem \ref{thm: eventual learn} is to explore the space-time structures of Markov chains. A brief review of the space-time structure of a Markov chain can be found in Appendix \ref{app: space-time}.
For fixed $N$ and $K$, we know that the state transition of our Markov chain is captured by its embedded {\em jump process}, denoted by $\pth{X^{J, N}(l): l\in \integers_+}$. Thus, the success event $E^N$ has an alternative representation: % which is slightly easier to analyze:
\begin{align}
\label{event: success repre}
E^N  \triangleq \sth{\lim_{t\diverge } X^N(t) = x^*} &= \sth{\lim_{l\diverge } X^{J, N}(l) = x^*} = \sth{\lim_{l\diverge } X_1^{J, N}(l) = N},
\end{align}
where $X_1^{J, N}(l)$ is the number of individuals that prefer the best option/arm at the $l^{th}$ jump. % Similar to $X_0^{N}(t)$, $X_0^{J, N}(l)$ is used to denote the number of individuals that temporarily do not have any prefernece over the $K$ arms.
Note that the coordinate process $\pth{X_1^{J, N}(l): \, l\in \integers_+}$ is a random walk\footnote{Indeed, all the $K+1$ coordinate processes are random walks.}; for each jump, $X_1^{J, N}$ either increases by one, decreases by one, or remains the same. Nevertheless, characterizing $\prob{E^N}$ with the representation in \eqref{event: success repre} is still highly nontrivial as the transition probabilities of the $K+1$ coordinate processes of $X^{J, N}(l)$ are highly correlated.
%Thus, we use an auxiliary event $F^N$,
%%because the probability of this auxiliary event is easier to bound.
%which is defined as the event that {\em all the $N$ individuals learn the best opinion without visiting any state in which NO individuals adopt the best opinion}, i.e.,
%%
%\begin{align}
%\label{event: simplified success}
%F^N &\triangleq  \sth{\lim_{l\diverge } X^{J, N}(l) = x^*} \cap \sth{X^{J, N}_1(l)\not=0 ~\forall ~ l} = \sth{\lim_{l\diverge } X^{J, N}_1(l) = N, ~ \&  ~ X^{J, N}_1(l)\not=0~\forall  ~ l}.
%\end{align}
%Note that $s\in \reals^{K+1}$, whose entries are labelled from zero, with $s_0$ being the number of individuals adopt no options and $s_1$ being the number of individuals adopt option $o_1$.
%
%It is easy to see that %for any $N$,
%\begin{align*}
%F^N \subseteq E^N, ~ \text{and}~ ~  \prob{F^N}\le \prob{E^N}, ~~ \forall ~ N.
%\end{align*}
%
% and their transition probabilities collectively depend on the states of $\pth{X^J(l): \, l \in \integers}$. % at the time when the $l^{th}$ jump occurs.
%\nb{LS: Is there any name for such a ``coordinate'' process???}
%
Fortunately, the {\em jump process} of the coordinate process $\pth{X_1^{J, N}(l): \, l\in \integers_+}$, denoted by $\pth{W(k): \, k\in \integers_+}$, is a random walk with the following nice property: For any $k$, as long as $W(k)\not=0$ and $W(k)\not=N$, the probability of moving up by one is at least $\frac{p_1}{p_1+p_2}$, and the probability of moving down by one is at most $\frac{p_2}{p_1+p_2}$.
%
% The following corollary follows directly form Propositions \ref{prop: embed con} and \ref{prop: embed dis}.
%
%
%
\begin{lemma}
\label{random walker}
For any $k$, given $W(k)\not=0$ and $W(k)\not=N$, then %for the random walk $\pth{W(k): \, k\in \integers_+}$, the probability of moving up by one is at least $\frac{p_1}{p_1+p_2}$, and the probability of moving down by one is at most $\frac{p_2}{p_1+p_2}$, i.e.,
\begin{align*}
  W(k+1)  =
  \begin{cases}
    W(k)+1, & \mbox{with probability at least  } \frac{p_1}{p_1+p_2} \\
     W(k)-1, & \mbox{otherwise}.
  \end{cases}
\end{align*}
\end{lemma}
%
%The proof of Lemma \ref{random walker} can be found in Appendix \ref{app: random walker}.
%
%\paragraph{Coupling of $\pth{W(k): \, k\in \integers_+}$ with a standard biased random walk}
With the property in Lemma \ref{random walker}, we are able to couple the embedded random walk with a standard biased random walk whose success probability is well understood.
Let $\pth{\widehat{W}(k): \, k\in \integers_+}$ be a random walk such that if $\widehat{W}(k) =0$ or $\widehat{W}(k) =N$, then $\widehat{W}(k+1) = \widehat{W}(k)$; otherwise,
\begin{align}
\label{aux: random walk}
\widehat{W}(k+1) =
\begin{cases}
\widehat{W}(k)+1 ~ & \text{with probability  } \frac{p_1}{ p_1+p_2}; \\
\widehat{W}(k)-1  ~ & \text{with probability  } \frac{p_2}{ p_1+p_2}.
\end{cases}
\end{align}
Intuitively, the embedded random walk $\pth{W(k): \, k\in \integers_+}$ has a higher tendency to move one step up (if possible) than that of the standard random walk \eqref{aux: random walk}. Thus, starting at the same position, the embedded random walk has a higher chance to be absorbed at position $N$ than that of the standard random walk. We justify this intuition through a formal coupling argument in Appendix \ref{sec: limit}.

For the standard random walk, the event $\sth{\lim_{k\diverge} \widehat{W}(k) =N  \mid \widehat{W}(0)=z_0}$ is referred to as the success probability of a gambler's ruin problem with initial wealth $z_0$. It is well-known that this probability increases geometrically with $z_0$.
\begin{proposition}\cite{hajek2015random}
\label{prop: gam ruin}
For any $z_0\in \integers_+$,
%\begin{align*}
%
$\prob{ \lim_{k\diverge} \widehat{W}(k) =N \mid \widehat{W}(0)=z_0} %= \frac{1-\pth{\frac{p_2}{p_1}}^{z_0}}{1-\pth{\frac{p_2}{p_1}}^{n}}
%  &= 1 - \frac{\pth{\frac{p_2}{p_1}}^{cn}\pth{1- \pth{\frac{p_2}{p_1}}^{(1-c)n}}}{1-\pth{\frac{p_2}{p_1}}^{n}}\\
   \ge 1- \pth{\frac{p_1}{p_2}}^{- z_0}.$
%\end{align*}
\end{proposition}
%With the coupling and Proposition \ref{prop: gam ruin} at hand, it is tempting to conclude that with high probability the embedded random walk will be absorbed in the desired state $x^*$.However,
Recall from \eqref{dl mc} that
$ X^N(0) = \qth{X_0^N(0), X_1^N(0), \cdots, X_K^N(0)} = \qth{N, 0, \cdots, 0},$
i.e., initially no individuals have any preference over the $K$ arms, and $X_k^N(0) =0$ for all $k=1, \cdots, K$. So the initial state of the embedded random walk $\pth{W(k): \, k\in \integers_+}$ is 0, i.e., $W(0)=0$. If we start coupling the embedded random walk and the standard random walk from time $t_c=0$, then $\widehat{W}(0)=X_1^N(t_c)=W(0)=0$ and the lower bound given in Proposition \ref{prop: gam ruin} is 0, which is useless. We need to find a proper coupling starting time such that with high probability, the position of the embedded random walk is sufficiently high. The next lemma says that $t_c=\frac{1}{\lambda}$ can be used as a good coupling starting time.
\begin{lemma}
\label{lm: initial wealth}
Let $t_c =\frac{1}{\lambda}$. For any $0<\delta <1$, with probability at least $1-  e^{-\frac{\mu p_1}{K e} \frac{\delta^2}{2} N }$, it holds that
$
X_1^N (t_c) \ge  (1-\delta)  \frac{\mu p_1}{K e} N.
$
\end{lemma}
The intuition behind Lemma \ref{lm: initial wealth} is that when $t$ is sufficiently small, successful memory state updates mainly rely on the uniform sampling rather than social sampling. Concretely, during a very short period, a few individuals have non-0 state memory, and it is highly likely that $c=0$ when $c$ is determined by social sampling. Thus, successful memory updates are likely to be {\em independent} of each other, and have some nice concentration properties.  %In addition, we know that the NULL arm $a_0$ will never generate reward one. So it is highly unlikely for individuals to successfully update their memory states via {\bf SocialObservation}. We justify this intuition in Appendix \ref{app: initial wealth}.

\vskip 0.6\baselineskip
By coupling the embedded random walk and the standard random walk at the random position $X_1^N (t_c)$, together with Proposition \ref{prop: gam ruin} and Lemma \ref{lm: initial wealth}, we are able to conclude Theorem \ref{thm: eventual learn}.
\end{proof}
\subsection{Transient System Behaviors}
In addition to {\em learnability}, it is also important to characterize the transient behavior of the learning dynamics in Algorithm \ref{alg: 1}, i.e., at a given time $t$, how many individuals prefer the best arm, the second best arm, etc. This subsection is devoted to characterize this transient system behaviors.

Let $\Delta^{K}$ be the simplex of dimension $K$, that is
%\begin{align}
%\label{def: simplex}
$\Delta^{K} \triangleq \sth{x\in \reals^{K+1}: ~\sum_{k=0}^{K} x_k =1, x_k\ge 0, ~ \forall ~ k}.$
%\end{align}
For any given $N$, let
\begin{align}
\label{def: scaled process}
Y^N(t) \triangleq \fr{X^N(t)}{N} \in \Delta^{K}, ~~~ \forall t\in \reals_+
\end{align}
be the scaled process of the original continuous-time Markov chain \eqref{dl mc}. We show that over an arbitrary and fixed finite time horizon $[0, T]$, with high probability (in $N$),  the fraction of individuals that prefer the best arm grows to 1 exponentially fast, which implies that the fractions of individuals that prefer the sub-optimal arms go to 0 exponentially fast.
\begin{theorem}
\label{thm: transient}
For a given $T \ge \bar{t}_c$, for any $0< \epsilon^{\prime} <  \lambda T \sqrt{K+1} \exp \qth{ \lambda \pth{5+\sqrt{K}}T }$, with probability at least $1-C_0\cdot e^{-N\cdot C_1}$, it holds that for all $\bar{t}_c \le t\le  T$,
%For any constant $c\in (0, 1)$, let $\bar{t}_c ~ \triangleq ~ \frac{\log \frac{1}{c}}{\lambda \frac{\mu}{K}\sum_{k^{\prime}=1}^K p_{k^{\prime}} }$.
\begin{align*}
1-Y_1^N(t) \le \exp \pth{-t\cdot R} +\epsilon^{\prime}, ~~ \text{and} ~~ Y_k^N(t) \le \exp \pth{-t\cdot R} +\epsilon^{\prime}, ~\forall ~ k\not=1,
\end{align*}
where $C_0= 2 \pth{K+1}$, $C_1=\frac{3-e}{9T \lambda} \frac{\pth{\epsilon^{\prime}}^2}{(K+1) e^{2 \lambda \pth{5+\sqrt{K}}T}}$, $R=\min \sth{\lambda \pth{p_1 - p_2}, ~  \lambda\pth{\frac{\mu}{K}+(1-\mu)}p_1}$, and $\bar{t}_c ~ \triangleq ~ \frac{\log \frac{1}{c}}{\lambda \frac{\mu}{K}\sum_{k^{\prime}=1}^K p_{k^{\prime}} }$ for any $c\in (0,1)$.
\end{theorem}
Typical sample paths of $Y^N$ are illustrated in Figure \ref{example}: When $N=200$, $K=2$, $\lambda=1$, $\mu=0.2$, $p_1=0.8$, and $p_2=0.4$, each of the component in $Y^N$ goes to their corresponding equilibrium states exponentially fast. In particular, these typical sample paths have two slightly different behaviors stages: At the first stage, $Y^N_1(t)$ and $Y^N_2(t)$ both increase up to the point where $Y^N_1(t)+Y^N_2(t) \approx 1$ -- noting that $Y^N_2(t)$ grows much slower than $Y^N_1(t)$. At the second stage, until entering their equilibrium states, $Y^N_1(t)$ is increasing and $Y^N_2(t)$ is decreasing. More importantly, $Y_0^N, Y_1^N,$ and $Y_2^N$ track their corresponding deterministic and smooth trajectories. 
\begin{figure}
 \centering
 \includegraphics[width=10cm]{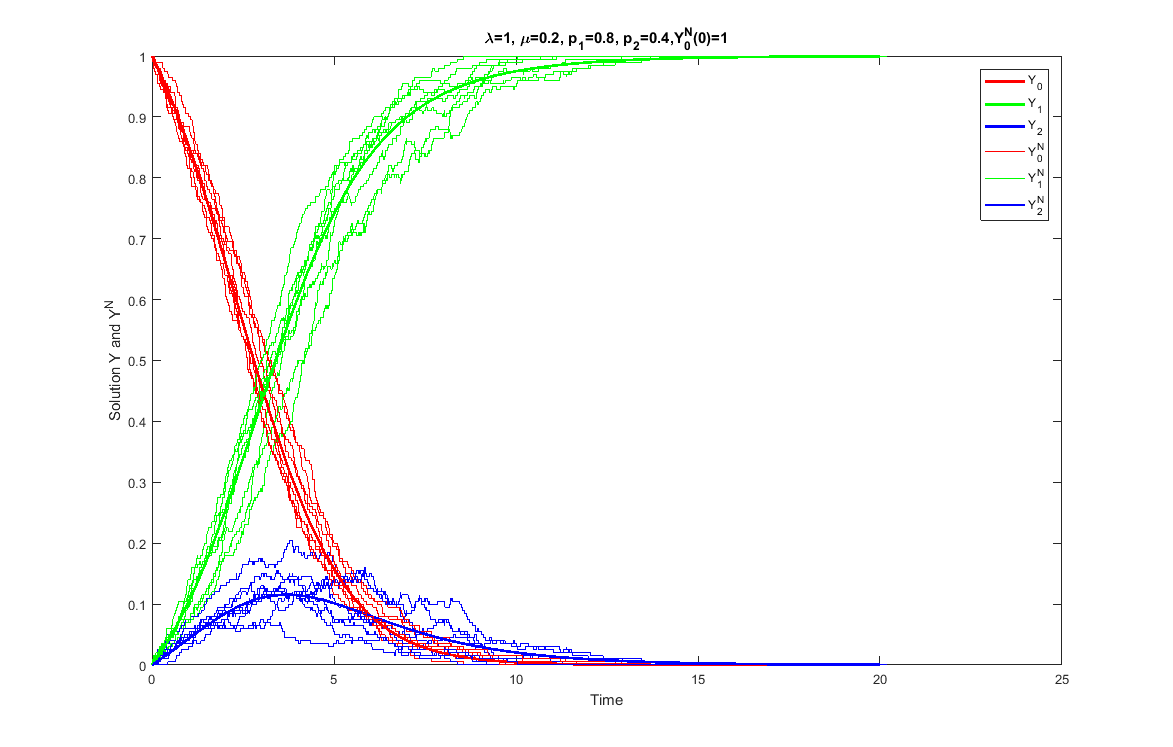}
 \vskip -1\baselineskip 
 \caption{}
 %\caption{$k=2$, $\mu=0.2$, $p_1=0.8$ and $p_2=0.4$ with initial state $[1, 0, 0]$; the trajectories of $Y_1(t)$ and $Y_2(t)$ have two slightly different behaviors stages: At the first stage, $Y_1(t)$ and $Y_2(t)$ both increase up to the point where $Y_1(t)+Y_2(t) \approx 1$ -- noting that $Y_2(t)$ grows much slower than $Y_1(t)$. At the second stage, until entering their equilibrium states, $Y_1(t)$ is increasing and $Y_2(t)$ is decreasing.}
 \label{example}
\end{figure}

For fixed $K$, $\lambda$, and $p_1$, $\lambda\pth{\frac{\mu}{K}+(1-\mu)}p_1$ is decreasing in $\mu\in (0, 1]$. As a result of this, $R$ is also decreasing in $\mu$. %When choosing $\mu$ to be sufficiently small, $R\approx \lambda \pth{p_1 - p_2}$.
On the other hand, Theorem \ref{thm: eventual learn} presents a lower bound on the success probability, i.e., $\prob{E^N}$, which is increasing in $\mu$.
It is unclear whether there is indeed a fundamental trade-off between the success probability $\prob{E^N}$ and the convergence rate $R$ or this is just an artifact of our analysis.
As can be seen later, $\bar{t}_c$ is introduced for a technical reason.
%
%
%\begin{remark}
%{\color{blue} Note that the exponential rate of convergence of the approximation, $C_1$, decreases as $T$ increases. This is not because the approximation becomes less accurate as time goes by (as it can be observed in Figure \ref{example}). It is due to the fact that we measure the distance between the sample paths of the stochastic process and the approximation using the supremum of the difference over the whole interval $[0,T]$. As a result, the longer the interval, the more likely it is that at some point the approximation will be more than $\epsilon$ away from the sample path.}
%\end{remark}

\begin{proof}[{\bf Proof Sketch of Theorem \ref{thm: transient}}]
We first show that for sufficiently large $N$, the scaled Markov chains $\pth{Y^N(t), t\in \reals_+}$ (defined in \eqref{def: scaled process}) can be approximated by the {\em unique} solution of an ODEs system. %Let $\ell \in \reals^{K+1}$.
Define $F: \Delta^{K} \to \reals^{K+1}$ as:
\begin{align}
\label{den function}
F(x) \triangleq \sum_{\ell \in \reals^{K+1}} \ell \cdot f(x, \ell),  ~~ \text{where }f\pth{x, \ell} \triangleq
\begin{cases}
%\lambda \pth{1 -y_0}y_0  p_0, ~~\text{if}~ \ell =e^0; \\
\lambda x_0 \pth{\frac{\mu}{K} +(1-\mu)x_k}p_k, ~~\text{if}~ \ell =e^k, ~ \forall~ k\not=0;\\
\lambda  x_{k^{\prime}} x_{k}  p_{k} , ~~\text{if}~ \ell = e^k - e^{k^{\prime}}   \text{for} ~ k^{\prime} \not=k, k^{\prime}\not=0;\\
0,  ~~~ \text{otherwise}.
\end{cases}
\end{align}
% where
% \begin{align}
% \label{function goal}
% f\pth{x, \ell} \triangleq
% \begin{cases}
% %\lambda \pth{1 -y_0}y_0  p_0, ~~\text{if}~ \ell =e^0; \\
% \lambda x_0 \pth{\frac{\mu}{K} +(1-\mu)x_k}p_k, ~~\text{if}~ \ell =e^k, ~ \forall~ k\not=0;\\
% \lambda  x_{k^{\prime}} x_{k}  p_{k} , ~~\text{if}~ \ell = e^k - e^{k^{\prime}}   \text{for} ~ k^{\prime} \not=k, k^{\prime}\not=0;\\
% 0,  ~~~ \text{otherwise}.
% \end{cases}
% \end{align}
Intuitively, $f(x, \ell)$ is the rate to move in direction $\ell$ from the current scaled state $x\in \Delta^K$, and function $F$ is the average movements of the scaled process $\pth{Y^N(t), t\in \reals_+}$.
 %Thus, the value $F(x)$ can be viewed as the average movement rate associated with of the current state $x$.
%The ODE system can be found in the following theorem.
%
\begin{lemma}
\label{thm: fluid model}
For a given $T$, as the group size $N$ increases, the sequence of scaled random processes $\pth{\frac{1}{N} X^N(t): t\in \reals_+}$ converge,  in probability, to a deterministic trajectory $Y(t)$ such that
\begin{align}
\label{eq: thm: ode}
\frac{\partial }{\partial t} Y(t) &= F \pth{Y(t)}, ~~~ t\in [0, T]
\end{align}
with initial condition $ Y(0) =\qth{1, 0, \cdots, 0}\in \Delta^{K}.$ In particular,
\begin{align*}
\prob{\sup_{0\le t \le T} \norm{Y^N(t) - Y(t)} \ge ~\epsilon^{\prime}  }\le  2 \pth{K+1} \exp \sth{-N \cdot C(\epsilon^{\prime})},
\end{align*}
where $ C(\epsilon^{\prime}) = \frac{3-e}{9T \lambda} \frac{\pth{\epsilon^{\prime}}^2}{(K+1) \exp \pth{2 \lambda \pth{5+\sqrt{K}}T}}.$
\end{lemma}
Lemma \ref{thm: fluid model} says that the scaled process $Y^N(t)$, with high probability, closely tracks a deterministic and smooth trajectory. This coincides with simulation in Figure \ref{example}. The approximation in Lemma \ref{thm: fluid model} is desired because the analysis of an ODEs system is relatively easier than that of the original stochastic system. Indeed, in a great variety of fields, such as biology, epidemic theory, physics, and chemistry \cite{kurtz1970solutions}, differential equations are used directly to model the {\em macroscopic} level system dynamics that are {\em arguably} caused by the {\em microscopic} level individuals interactions in the system.

For any $t\in [0, T]$, the ODEs system $\frac{\partial }{\partial t} Y(t) = F \pth{Y(t)}$ in \eqref{eq: thm: ode} can be written out as:
%For the NULL arm $a_0$,
\begin{align}
\label{limit 0}
\dot{Y_0}(t) %&= (1-Y_0(t)) \lambda Y_0(t) p_0 - Y_0(t) \lambda \sum_{k^{\prime}=1}^K \pth{\frac{\mu}{k} +(1-\mu)Y_{k^{\prime}}(t)} p_j\\
&= - Y_0(t) \lambda \frac{\mu}{K}\sum_{k=1}^k p_k - Y_0(t) \lambda \sum_{k=1}^K (1-\mu) p_k Y_k(t),\\
\dot{Y_k}(t)  %& =  Y_0(t) \lambda  \pth{\frac{\mu}{K} +(1-\mu)Y_k(t)} p_k + Y_k(t) \lambda p_{k}\sum_{k^{\prime}: k^{\prime}\not=k, 1\le k^{\prime}\le K} Y_{k^{\prime}}(t)  - Y_k(t)\lambda \sum_{k^{\prime}: k^{\prime}\not=k, 1\le k^{\prime}\le K} Y_{k^{\prime}}(t) p_{k^{\prime}} \\
&=  Y_0(t) \lambda  \frac{\mu}{K} p_k + Y_k(t) \lambda \pth{(1-\mu)p_kY_0(t)+ \sum_{k^{\prime}=1}^K (p_k-p_{k^{\prime}})Y_{k^{\prime}}(t)}, ~~\forall k=1, \cdots, K.
\label{limit 1}
\end{align}
Note that our ODE system is quite different from the antisymmetric Lotka-Volterra equation \cite{knebel2015evolutionary}. The Lotka-Volterra equation is a typical replicator dynamics, where if $Y_k(0)=0$ for some $k$, it remains to be zero throughout the entire process. In contrast, in our system, the desired ${\bf Y}^* =\qth{Y_0^*, Y_1^*, \cdots, Y_K^*} = \qth{0, 1, 0, \cdots, 0}$ is achievable with $Y(0) = \qth{1, 0, \cdots, 0}$, and can be achieved exponentially fast. 
%formally stated in the next theorem.
\begin{lemma}
\label{lemma: uniqueness}
With initial state $ Y(0) =\qth{1, 0, \cdots, 0}$, the state ${\bf Y}^*  = \qth{0, 1, 0, \cdots, 0}$ is achievable.
\end{lemma}
%In addition, $Y(t)$ converges to the desired ${\bf Y}^*$ exponentially fast.
\begin{lemma}
\label{thm: convergence rate ode}
For any constant $c\in (0, 1)$, let $\bar{t}_c ~ \triangleq ~ \frac{\log \frac{1}{c}}{\lambda \frac{\mu}{K}\sum_{k^{\prime}=1}^K p_{k^{\prime}} }$.
% \begin{align}
% \label{eq: reference starting point}
% \bar{t}_c ~ \triangleq ~ \frac{\log \frac{1}{c}}{\lambda \frac{\mu}{K}\sum_{k^{\prime}=1}^K p_{k^{\prime}} }.
% \end{align}
%
When $t\ge \bar{t}_c$, $Y(t)$ converges to ${\bf Y}^* = \qth{0, 1, 0, \cdots, 0}$ exponentially fast, with rate $\min \sth{\lambda \pth{p_1 - p_2}, ~  \lambda\pth{\frac{\mu}{K}+(1-\mu)}p_1}$.
\end{lemma}

By triangle inequality,  
\begin{align*}
1-Y_1^N(t) \le  \norm{Y^N(t) -Y(t)} + \abth{1-Y_1(t)}, ~ \text{and}~Y_k^N(t) \le  \norm{Y^N(t) -Y(t)} + \abth{Y_k(t)}, ~~~ \forall k\not=1.
\end{align*}
Thus, Theorem \ref{thm: transient} follows immediately from Lemmas \ref{thm: fluid model}, \ref{lemma: uniqueness},and \ref{thm: convergence rate ode}.
\end{proof}

\section{Concluding Remarks}
\label{subsec: comparison}
We studied the collaborative multi-armed bandit problems in social groups wherein each individual suffers finite memory constraint \cite{1054427}. % and shares the common goal of learning the best arm/option.
We rigorously investigated the power of persuasion, i.e., communication, in improving an individual's learning capability.
%In contrast to isolated individuals \cite{1054427} for whom learning the best option is impossible, we showed that with the aid of social persuasion, learning the best option becomes possible. We employed the mean-field approximation method to deal with the non-trivial stochastic dependency arising from this social interaction. The mean-field approximation method might be useful for other stochastic distributed algorithms that might arise in settings where $N$ is sufficiently large, such as insect colonies, systems of wireless devices, and population protocols.
%
%The problem of multi-armed bandit in social groups is originally proposed in \cite{Celis:2017:DLD:3087801.3087820}, where
% 
Similar learning dynamics are considered in \cite{Celis:2017:DLD:3087801.3087820} with the following fundamental differences. %Nevertheless, our analysis is completely different from that in \cite{Celis:2017:DLD:3087801.3087820}. We generalize the results in \cite{Celis:2017:DLD:3087801.3087820} and complete their analysis as follows:
\begin{itemize}
	\item We relax their synchronization assumption. It is assumed in \cite{Celis:2017:DLD:3087801.3087820} that time is slotted, and all individuals attempt to make one-step update simultaneously, where the updates are based on the system's state at the end of previous round. This synchronization assumption imposed additional implementation challenges: Synchronization does not come for free. Typically the cost of synchronization depends deterministically on the slowest individual in the system. %(2) Additional memory requirements: An individual, say $i$, has to wait until the slowest individual has chosen one peer to observe before the individual $i$ can proceed to update its state. In the meanwhile, individual $i$ has to temporarily store the state of its observed peer -- inducing additional memory cost.

	In contrast, our work relaxes this synchronization assumption and considers the less restrictive asynchronous setting where each individual has an independent Poisson clock, and attempts to perform a one-step update immediately when its local clock ticks. The Poisson clock model is very natural: The real world is an asynchronous one and there are physical reasons to use a modeling approach based on Poisson clocks \cite{ross2014introduction}. This model avoids the implementation challenges in \cite{Celis:2017:DLD:3087801.3087820}, nevertheless, it causes non-trivial analysis challenge -- we have to deal with the stochastic dependency among any updates. In contrast, with synchronization, the individuals' updates of the same round in \cite{Celis:2017:DLD:3087801.3087820} are conditionally independent.
	
	\item We relax the requirement of performing uniform sampling by all individuals, and we are able to show learnability of {\em every} individual. It is assumed in \cite{Celis:2017:DLD:3087801.3087820} that, in each round, every individual, regardless of its preference states, performs uniform sampling with probability $\mu$. This assumption is imposed for a technical reason: They wanted to {\em ``ensure that the population does not get stuck in a bad option''}. However, as a result of their assumption, as $t\diverge$, there is a constant fraction $\mu\in (0, 1]$ of individuals that cannot learn the best option.  In contrast, in our learning dynamics, we require only that the individuals without any preference do uniform sampling. We overcome their concerns of ``get stuck in a bad option'' by showing that such events occur with probability diminishing to $0$ as the social group size $N\diverge$. % every individual learns the best option. %s will NOT get stuck in those undesired scenarios. % and, as $t\diverge$, all

	\item We use the mean-field approximation method to provide a provable connection between the finite and the infinite population dynamics. 	In \cite{Celis:2017:DLD:3087801.3087820}, the authors first define an ``infinite population'' version of their dynamics and then translate its convergence properties to the finite population stochastic dynamics. Unfortunately, the connection between their infinite and the finite population dynamics is only established through the {\em ``non-rigorous thought process''} -- as the authors themselves commented \cite{Celis:2017:DLD:3087801.3087820}. Similar heuristic arguments were also made in the evolutionary biology literature \cite{kang2016dynamical}. In contrast, we use the mean-field approximation method to provide a provable connection between these two dynamics. %We show that in our problem the sample paths of a properly scaled sequence of stochastic processes, as $N\diverge$, concentrate around the unique solution of an ODE system \cite{kurtz1971limit, shwartz1995large}. %Such an ODE system is often referred to as a {\em mean-field} model, and the corresponding approximation is referred to as a {\em mean-field approximation} method.

\end{itemize}

\section*{Acknowledgements}
We would like to thank John N.\ Tsitsiklis for valuable discussions and comments.

% On the technical side, we use an interesting combination of two techniques.
% First, we showed that with  probability $\to 1$ as the social group size $N \diverge $, every individual in the social group learns the best option. To show this, we explored the space-time structure of a Markov chain and relied on a coupling argument to conclude learnability. Notably, this technique might be of independent interest since it enables us to deal with the existence of absorbing states of a Markov chain in interchanging the limits of $N\diverge$ and $t\diverge$.
% %
% %
% Note that the ``learnability" under discussion is a time-asymptotic notion -- recalling that we say an individual learns the best option if, as $t\diverge$, it only pulls the option with the highest average reward. In addition to ``learnability", it is also important to characterize the transient behavior of the learning dynamics, i.e., at a given time $t$, how many individuals prefer the best option, the second best option, etc.  The transient behavior over finite $[0, T]$ is harder to analyze directly; for this, we use the {\em mean-field approximation} method.

%There is a number of potentially interesting extensions of our formulations, such as alternative reward distributions
%\begin{itemize}
%\item More general reward distributions: gaussian ...
%\item Communication constraints: communication networks.
%\end{itemize}

\bibliographystyle{acm}
\bibliography{alpha}

\appendix

\section{Space-Time Structure}
\label{app: space-time}
 Recall that a time-homogeneous, discrete-state Markov chain (can possibly be continuous-time or discrete-time) can be described alternatively by specifying its space-time structure \cite{hajek2015random}, which is simply the sequence of states visited (jump process) and how long each state is stayed at per visit (holding times).

Let $H_{l}$ be the time that elapses between the $l^{th}$ and the ${l+1}^{th}$ jumps of $\pth{X^N(t): t\in \reals_+}$. Intuitively, $H_l$ is the amount of time that the process is ``held'' by the $l^{th}$ jump. More formally,
\begin{align}
\label{space-time holding 0}
H_0 & = \min_t \sth{t\ge 0: ~ X^N(t)\not=X^N(0)} \\
\label{space-time holding l}
H_l & = \min_t \sth{t\ge 0: ~ X^N(H_0+ \cdots + H_{l-1} + t)\not= X^N(H_0+ \cdots + H_{l-1})},
\end{align}
as the $l^{th}$ jump occurs at time $H_0+ \cdots + H_{l-1}$ and the ${l+1}^{th}$ jump occurs at time $H_0+ \cdots + H_{l}$.
Correspondingly, the {\em jump process} embedded in the continuous-time Markov chain $X^N(t)$ is defined by
\begin{align}
\label{space-time jump process}
  X^{J, N}(0)  = X^N(0) ~~ ~~  \text{and} ~~ ~~   X^{J, N}(l)  = X^N(H_0+ \cdots + H_{l-1}),
\end{align}
with
\begin{align}
\label{jump entries}
  X^{J, N}(l) &  =\qth{X^{J, N}_0(l), X^{J, N}_1(l), \cdots, X^{J, N}_K(l)},
\end{align}
where $X^{J, N}_k(l),$ is the number of individuals that prefer arm $a_k$ at the $l^{th}$ jump.
The holding times $\{H_l\}_{l=0}^{\infty}$ and the jump process $\pth{X^{J, N} (l): l\in \integers_+}$ contain all the information needed to reconstruct the original Markov chain $X^N(t)$, and vice versa \cite{hajek2015random}. %In this work, we are interested in the space-time structure because the joint distribution of $\{H_l\}_{l=0}^{\infty}$ and $\pth{X^{J, N} (l): l\in \integers_+}$ has the following important properties that can be explored to study the time asymptotical behaviors of our learning dynamics.

\begin{proposition}\cite{hajek2015random}
\label{prop: embed con}
Let $\pth{X(t): t\in \reals_+}$ be a time-homogeneous, pure-jump Markov process with generator matrix $Q$. Then
the jump process $X^J$ is a discrete-time, time-homogeneous Markov process, and its one-step transition probabilities are given by
     \begin{align}\label{jump: one-step}
       p_{ss^{\prime}}^J = \begin{cases}
                    -\frac{q_{ss^{\prime}}}{q_{ss}}, & \mbox{for } s\not=s^{\prime}; \\
                    0, & \mbox{for } s=s^{\prime},
                  \end{cases}
     \end{align}
     with the convention that $\frac{0}{0} \triangleq 0$. %Moreover, given $X(0)$, $X^J(1)$ is conditionally independent of $H_0$.
% \begin{enumerate}
%   \item The jump process $\pth{X^J(l): l\in \integers_+ }$ is a discrete-time, time-homogeneous Markov process, and its one-step transition probabilities are given by
%       \begin{align}\label{jump: one-step}
%         p_{ij}^J = \begin{cases}
%                      -\frac{q_{ij}}{q_{ii}}, & \mbox{for } i\not=j; \\
%                      0, & \mbox{for } i=j,
%                    \end{cases}
%       \end{align}
%       with the convention that $\frac{0}{0} \triangleq 0$.
%   \item Given $X(0)$, $X^J(1)$ is conditionally independent of $H_0$.
%   \item Given $X^J(0)=j_0, \cdots, X^J(l)=j_k$, the variables $H_0, \cdots, H_k$ are conditionally independent, and the conditional distribution of $H_l$ is exponential with parameter $-q_{j_l j_l}$.
% \end{enumerate}
%\nb{LS: Do we really need 2 and 3???}
\end{proposition}
%

% \nb{LS: cast some interpretation of the above proposition + introduce the next prop}
Our analysis will also use the space-time structure of the discrete-time Markov chain, as we will use such structure of the above derived jump process $X^J$.
\begin{proposition}\cite{hajek2015random}
\label{prop: embed dis}
Let $\pth{X(k): k\in \integers_+}$ be a time-homogeneous Markov process with one-step transition probability matrix $P$. Then
the jump process $\pth{X^J(l): l\in \integers_+}$ is itself a time-homogeneous Markov process, and its one-step transition probabilities are given by
      \begin{align}\label{jump: one-step}
        p_{ss^{\prime}}^J = \begin{cases}
                     \frac{p_{ss^{\prime}}}{1-p_{ss}}, & \mbox{for } s\not=s^{\prime}; \\
                     0, & \mbox{for } s=s^{\prime},
                   \end{cases}
      \end{align}
      with the convention that $\frac{0}{0} \triangleq 0$.
% \begin{enumerate}
%   \item The jump process $\pth{X^J(l): l\in \integers_+}$ is itself a time-homogeneous Markov process, and its one-step transition probabilities are given by
%       \begin{align}\label{jump: one-step}
%         p_{ss^{\prime}}^J = \begin{cases}
%                      \frac{p_{ss^{\prime}}}{1-p_{ss}}, & \mbox{for } s\not=s^{\prime}; \\
%                      0, & \mbox{for } s=s^{\prime},
%                    \end{cases}
%       \end{align}
%       with the convention that $\frac{0}{0} \triangleq 0$.
%       \item Given $X(0)$, $X^J(1)$ is conditionally independent of $H_0$.
%   \item Given $X^J(0)=j_0, \cdots, X^J(l)=j_k$, the variables $H_0, \cdots, H_k$ are conditionally independent, and the conditional distribution of $H_l$ is geometric with parameter $p_{j_l j_l}$.
% \end{enumerate}
\end{proposition}

% Using Propositions \ref{prop: embed con} and \ref{prop: embed dis}, we are able to show that with probability approaching one as $n \diverge$, all individuals in the social group eventually learn the best option.

\section{Learnability}
\label{sec: limit}

\subsection{Proof of Lemma \ref{random walker}}
\label{app: random walker}
The proof of Lemma \ref{random walker} uses the following claim.
\begin{claim}
\label{clm: bounding}
For any $c \ge 0$, $y>0$, $x+y >0$ and $x+y+c >0$, it holds that
\begin{align*}
  \frac{x}{x+y} & <\frac{x+c}{x+c+y}.
\end{align*}
\end{claim}
\begin{proof}
The proof of this claim is elementary, and is presented for completeness.
\begin{align*}
  \frac{x+c}{x+c+y} - \frac{x}{x+y} & = 1-\frac{y}{x+c+y} -1 + \frac{y}{x+y} = y\pth{\frac{1}{x+y} - \frac{1}{x+c+y}} \ge 0.
\end{align*}
\end{proof}

Now we are ready to prove Lemma \ref{random walker}.
\begin{proof}[Proof of Lemma \ref{random walker}]
For ease of exposition, for a fixed $k\in \integers_+$, define
\begin{align}
\label{rw: moving up}
A^k \triangleq \sth{\omega : W(k+1) = W(k)+1 ~ \text{given }~W(k)\notin \{0, N\}}.
\end{align}
To show Lemma \ref{random walker}, it is enough to show $\prob{A^k} \ge \frac{p_1}{p_1+p_2}$.

We first link the random walk back to the first order space-time structure of the original continuous-time Markov chain as follows:
\begin{align*}
A^k &=\sth{\omega : W(k+1) = W(k)+1 ~ \text{given }  W(k)\notin \{0, N\}}\\
&= \cup_{l=1}^{\infty} \Big{\{}\omega : \text{the $k+1^{th}$ move of $W$ occurs at the ${l+1}^{th}$ jump of $X^{J, N}$ }\\
 & ~~~~~~~~~~~~~~~ \& ~ X^{J, N}_1(l+1) =  X^{J, N}_1(l)+1  ~ \text{given }W(k)\notin \{0, N\}  ~~~~ ~~~~  \Big{\}}.
\end{align*}
For ease of exposition, define $B^k_l$ as
\begin{align}
\label{def: bkl}
\nonumber
B^k_l &\triangleq \Big{\{}\omega : \text{the $k+1^{th}$ move of $W$ occurs at the ${l+1}^{th}$ jump of $X^{J, N}$ }\\
\nonumber
 & ~~~~~~~~~ \& ~ X^{J, N}_1(l+1) =  X^{J, N}_1(l)+1 ~ \text{given }W(k)\notin \{0, N\}  ~~~~ ~~~~  \Big{\}}\\
 \nonumber
 & = \Big{\{}\omega : \text{the $k+1^{th}$ move of $W$ occurs at the ${l+1}^{th}$ jump of $X^{J, N}$ }\\
 & ~~~~~~~~~ \& ~ X^{J, N}_1(l+1) =  X^{J, N}_1(l)+1 ~ \text{given }X^{J, N}_1(l)\notin \{0, N\}  ~~~~ ~~~~  \Big{\}}.
\end{align}
It is easy to see that $ B^k_l  = \O ~\text{for } l < k$. So we get $ A^k = \cup_{l \ge k} ~ B^k_l. $
In addition, by definition, $B^k_l \cap B^k_{l^{\prime}} =\O, \forall ~ l\not= l^{\prime}$. Thus,
\begin{align}
\label{eq: aaa}
\prob{A^k} = \prob{\cup_{l\ge k} B_l^k} = \sum_{l\ge k} \prob{B_l^k}.
\end{align}
Now we focus on $\prob{B_l^k}$. Let $\calS^N$ be the collection of valid states such that for every $s\in \calS^N$, $s_1\notin \{0, N\}$ --  recalling that a valid state is a partition of integer $N$.  By total probability argument, we have
\begin{align}
\label{eq: total prob}
\prob{B_l^k} & =  \sum_{s\in \calS^N} \prob{X^{J, N}(l)=s \mid X^{J, N}_1(l)\notin \{0, N\}} \prob{B_l^k \mid X^{J, N}(l)=s},
\end{align}
where
\begin{align}
\label{eq: b total}
\sum_{s\in \calS^N} \prob{X^{J, N}(l)=s \mid X^{J, N}_1(l)\notin \{0, N\}} =1.
\end{align}
Define event $C_l^k$ as follow:
\begin{align}
\label{def: ckl}
C^k_l \triangleq \sth{\omega : \text{the $k+1^{th}$ move of $W$ occurs at the ${l+1}^{th}$ jump of $X^{J, N}$ given $X^{J, N}_1(l)\notin \{0, N\}$}}.
\end{align}
It is easy to see that for a fixed $k$,
\begin{align}
\label{eq: c total}
\sum_{l\ge k} \prob{C^k_l } = 1.
\end{align}
For $l\ge k$, we have
\begin{align}
\label{eq:bbb}
&\prob{B_l^k \mid X^{J, N}(l)=s}=  \prob{C_l^k \mid X^{J, N}(l)=s} \prob{X^{J, N}_1(l+1) =  X^{J, N}_1(l)+1 \mid C_l^k, X^{J, N}(l)=s}.
\end{align}
%\begin{align*}
%&\prob{B_l^k} =  \sum_{s\in \calS^N} \prob{X^{J, N}(l)=s}\prob{B_l^k \mid X^{J, N}(l)=s}\\
%&=  \sum_{s\in \calS^N}  \prob{X^{J, N}(l)=s} \prob{C_l^k \mid X^{J, N}(l)=s} \prob{X^{J, N}_1(l+1) =  X^{J, N}_1(l)+1 \mid  X^{J, N}(l)=s, C_l^k}.
%\end{align*}
We claim that
\begin{align}
\label{eq: component bound}
 \prob{X^{J, N}_1(l+1) =  X^{J, N}_1(l)+1 \mid C_l^k, X^{J, N}(l)=s}\ge \frac{p_1}{p_1+p_2}.
\end{align}
We postpone the proof of \eqref{eq: component bound} to the end of the proof of Lemma \ref{random walker}.
With \eqref{eq: component bound}, we are able to conclude that $\prob{A^k} \ge \frac{p_1}{p_1+p_2}$. In particular, equation \eqref{eq:bbb} becomes
\begin{align}
\label{eq: ccc}
\prob{B_l^k \mid X^{J, N}(l)=s}
&\ge  \prob{C_l^k \mid X^{J, N}(l)=s} \frac{p_1}{p_1+p_2}.
\end{align}
By \eqref{eq: aaa}, \eqref{eq: total prob} and \eqref{eq: ccc}, we have
\begin{align*}
\prob{A^k}  &= \sum_{l\ge k}   \sum_{s\in \calS^N} \prob{X^{J, N}(l)=s \mid X^{J, N}_1(l)\notin \{0, N\}} \prob{B_l^k \mid X^{J, N}(l)=s}\\
&\ge   \sum_{l\ge k} \sum_{s\in \calS^N} \prob{X^{J, N}(l)=s \mid X^{J, N}_1(l)\notin \{0, N\}} \prob{C_l^k \mid X^{J, N}(l)=s} \frac{p_1}{p_1+p_2} \\
%& = \frac{p_1}{p_1+p_2}\sum_{l\ge k} \sum_{s\in \calS^N} \prob{X^{J, N}(l)=s \mid X^{J, N}_1(l)\notin \{0, N\}} \prob{C_l^k} \\
& = \frac{p_1}{p_1+p_2} \sum_{l\ge k} \sum_{s\in \calS^N} \prob{X^{J, N}(l)=s, C_l^k}\\
& = \frac{p_1}{p_1+p_2} \sum_{l\ge k} \sum_{s\in \calS^N} \prob{C_l^k} \prob{X^{J, N}(l)=s \mid C_l^k}\\
& = \frac{p_1}{p_1+p_2},
\end{align*}
where the last equality follows from \eqref{eq: b total} and \eqref{eq: c total}.

Next we prove \eqref{eq: component bound}. We have
\begin{align}
\label{eq: ddd}
\nonumber
 &\prob{X^{J, N}_1(l+1) =  X^{J, N}_1(l)+1 \mid C_l^k, X^{J, N}(l)=s} \\
 \nonumber
% & = \prob{X^{J, N}_1(l+1) =  X^{J, N}_1(l)+1 \mid C_l^k, X^{J, N}(l)=s, s_1\not=0, N}\\
& = \mathbb{P} \big{\{}X^{J, N}_1(l+1) =  X^{J, N}_1(l)+1 \mid X^{J, N}(l)=s, s_1\notin \{0, N\}, \text{the $k+1^{th}$ move of $W$}\\
\nonumber
 &~~~~~~~~~~~~~~~~~~~ ~~~~~~~~~~~~ ~~~~~~~~~~~~~ \text{occurs at the $l+1^{th}$ jump of $X^{J, N}$ given $W(k)\notin \{0, N\}$}\big{\}}\\
 \nonumber
 & = \mathbb{P} \big{\{} X^{J, N}_1(l+1) =  X^{J, N}_1(l)+1 \mid X^{J, N}(l)=s, s_1\notin \{0, N\} \\
 \nonumber
 &~~~~~~~~~~~~~~~~~~~ ~~~~~~~~~~~~ ~~~~~~~~~~~~~ \text{one move of $W$ occurs at the $l+1^{th}$ jump of $X^{J, N}$},\\
 \nonumber
 &~~~~~~~~~~~~~~~~~~~ ~~~~~~~~~~~~ ~~~~~~~~~~~~~ \text{and there are $k$ moves of $W$ occur among the first $l$ jumps of $X^{J, N}$}\big{\}}\\
 \nonumber
 & \overset{(a)}{=} \mathbb{P} \big{\{} X^{J, N}_1(l+1) =  X^{J, N}_1(l)+1 \mid X^{J, N}(l)=s, s_1\notin \{0, N\} \\
 \nonumber
 &~~~~~~~~~~~~~~~~~~~ ~~~~~~~~~~~~ ~~~~~~~~~~~~~ \text{one move of $W$ occurs at the $l+1^{th}$ jump of $X^{J, N}$}\big{\}} \\
 & = \frac{\prob{X^{J, N}_1(l+1) =  X^{J, N}_1(l)+1 \mid X^{J, N}(l)=s, s\in \calS^N}}{\prob{\text{one move of $W$ occurs at the $l+1^{th}$ jump of $X^{J, N}$} \mid X^{J, N}(l)=s, s\in \calS^N}},
\end{align}
where equality (a) follows from the Markov property of $X^{J, N}$.
%Let $s\in \calS^N$.
By Proposition \ref{prop: embed con}, we know
\begin{align}
\label{eq: eee}
\prob{X^{J, N}_1(l+1) =  X^{J, N}_1(l)+1 \mid X^{J, N}(l)=s, s\in \calS^N} = - \sum_{s^{\prime}: s_1^{\prime} = s_1+1} \frac{q_{s, s^{\prime}}}{q_{s,s}}.
\end{align}
Note that $\sum_{s^{\prime}: s_1^{\prime} = s_1+1}q_{s, s^{\prime}}$ is exactly the birth rate of the best arm, i.e., arm $a_1$. That is,
\begin{align}
\label{eq: eee 1}
\sum_{s^{\prime}: s_1^{\prime} = s_1+1}q_{s, s^{\prime}} = s_0\lambda \pth{\frac{\mu}{K} +\pth{1-\mu}\frac{s_1}{N}} p_1 + \sum_{j\ge 2} s_j\lambda \frac{s_1}{N} p_1.
\end{align}
Similarly,
\begin{align}
\label{eq: fff}
\prob{\text{one move of $W$ occurs at the $l+1^{th}$ jump of $X^{J, N}$} \mid X^{J, N}(l)=s, s\in \calS^N} & = -\sum_{s^{\prime}: s_1^{\prime} = s_1\pm 1} \frac{q_{s, s^{\prime}}}{q_{s,s}},
\end{align}
and $\sum_{s^{\prime}: s_1^{\prime} = s_1\pm 1} q_{s, s^{\prime}}$ is the summation of birth rate and death rate of the best arm.  
Specifically,
\begin{align}
\label{eq: fff 1}
\sum_{s^{\prime}: s_1^{\prime} = s_1\pm 1} q_{s, s^{\prime}}= s_0\lambda \pth{\frac{\mu}{K} +\pth{1-\mu}\frac{s_1}{N}} p_1 + \sum_{j\ge 2} s_j\lambda \frac{s_1}{N} p_1 + s_1\lambda \sum_{j\ge 2}\frac{s_j}{N} p_j.
\end{align}
Thus, by \eqref{eq: ddd}, \eqref{eq: eee} and \eqref{eq: fff}, we have
\begin{align*}
\prob{X^{J, N}_1(l+1) =  X^{J, N}_1(l)+1 \mid C_l^k, X^{J, N}(l)=s} &=  \frac{ - \sum_{s^{\prime}: s_1^{\prime} = s_1+1} \frac{q_{s, s^{\prime}}}{q_{s,s}}}{-\sum_{s^{\prime}: s_1^{\prime} = s_1\pm 1} \frac{q_{s, s^{\prime}}}{q_{s,s}}} = \frac{  \sum_{s^{\prime}: s_1^{\prime} = s_1+1} q_{s, s^{\prime}}}{\sum_{s^{\prime}: s_1^{\prime} = s_1\pm 1} q_{s, s^{\prime}}},
\end{align*}
and by \eqref{eq: eee 1} and \eqref{eq: fff 1}, we have
\begin{align}
\label{eq: ggg}
\nonumber
&\prob{X^{J, N}_1(l+1) =  X^{J, N}_1(l)+1 \mid C_l^k, X^{J, N}(l)=s}\\
\nonumber
%\frac{  \sum_{s^{\prime}: s_1^{\prime} = s_1+1} q_{s, s^{\prime}}}{\sum_{s^{\prime}: s_1^{\prime} = s_1\pm 1} q_{s, s^{\prime}}}=
&=\frac{s_0\lambda \pth{\frac{\mu}{K} +\pth{1-\mu}\frac{s_1}{N}} p_1 + \sum_{j\ge 2} s_j\lambda \frac{s_1}{N} p_1}{s_0\lambda \pth{\frac{\mu}{K} +\pth{1-\mu}\frac{s_1}{N}} p_1 + \sum_{j\ge 2} s_j\lambda \frac{s_1}{N} p_1 + s_1\lambda \sum_{j\ge 2}\frac{s_j}{N} p_j}\\
\nonumber
& = \frac{s_0  \pth{\frac{\mu}{K} + (1-\mu)\frac{s_1}{N}}p_1 + (N-s_1-s_0)  \frac{s_1}{N} p_1}{s_0 \pth{\frac{\mu}{K} + (1-\mu)\frac{s_1}{N}}p_1 + (N-s_1-s_0)  \frac{s_1}{N} p_1+s_1 \sum_{j\ge 2} \frac{s_j}{N}p_j}\\
\nonumber
& \overset{(a)}{\ge } \frac{s_0  (1-\mu)\frac{s_1}{N}p_1 + (N-s_1-s_0)  \frac{s_1}{N} p_1}{s_0 (1-\mu)\frac{s_1}{N}p_1 + (N-s_1-s_0)  \frac{s_1}{N} p_1+s_1 \sum_{j\ge 2} \frac{s_j}{N}p_j}\\
&= \frac{(N-s_1 -\mu s_0)p_1}{(N-s_1 -\mu s_0)p_1 + \sum_{j\ge 2} s_jp_j },
\end{align}
where inequality (a) follows from Claim \ref{clm: bounding}. In addition, we have
\begin{align*}
  \sum_{j\ge 2} s_jp_j & \le \sum_{j\ge 2} s_jp_2  = (N-s_0 - s_1) p_2 \le (N-\mu s_0 -  s_1) p_2.
\end{align*}
So \eqref{eq: ggg} becomes
\begin{align*}
\prob{X^{J, N}_1(l+1) =  X^{J, N}_1(l)+1 \mid C_l^k, X^{J, N}(l)=s}
& >  \frac{(N-s_1 -\mu s_0)p_1}{(N-s_1 -\mu s_0)p_1 + (N-\mu s_0 -  s_1) p_2}\\
& = \frac{p_1}{p_1+p_2},
\end{align*}
proving \eqref{eq: component bound}.

\vskip \baselineskip

Therefore, the proof of Lemma \ref{random walker}  is complete.

\end{proof}

%The proof of Lemma \ref{random walker} can be found in Appendix \ref{app: random walker}.

%\paragraph{Coupling of $\pth{W(k): \, k\in \integers_+}$ with a standard biased random walk}
%

\subsection{Coupling}
Recall that random walk $\pth{\widehat{W}(k): \, k\in \integers_+}$ defined in \eqref{aux: random walk}: If $\widehat{W}(k) =0$ or $\widehat{W}(k) =N$, then $\widehat{W}(k+1) = \widehat{W}(k)$; Otherwise,
\begin{align*}
%\label{aux: random walk}
\widehat{W}(k+1) =
\begin{cases}
\widehat{W}(k)+1 ~ & \text{with probability  } \frac{p_1}{ p_1+p_2}; \\
\widehat{W}(k)-1  ~ & \text{with probability  } \frac{p_2}{ p_1+p_2}.
\end{cases}
\end{align*}
Intuitively, the embedded random walk has a higher tendency to move one step up (if possible) than that of the standard random walk \eqref{aux: random walk}. Thus, starting at the same position, the embedded random walk should have a higher chance to be absorbed at position $N$. Formal coupling argument is given below.
 
%
%Recall that $\pth{W(k): \, k\in \integers_+}$ is the second-order embedded Markov chain. That is, it is the embedded chain of the embedded chain $\pth{X^J(l): \, l\in \integers_+}$ of the original continuous-time Markov chain $\pth{X^N(t): \, l\in \reals_+}$.
From Propositions \ref{prop: embed con}  and \ref{prop: embed dis}, we know that the transition probability of the embedded random walk $\pth{W(k): \, k\in \integers_+}$ is determined by the entire state (which is random) of the original continuous-time Markov chain $\pth{X^N(t): \, t\in \reals_+}$ right proceeding the $k^{th}$ move of $\pth{W(k): \, k\in \integers_+}$.

Let $x^N(\cdot) = \qth{x_0^N(\cdot), x_1^N(\cdot), \cdots, x_K^N(\cdot)}$ be an arbitrary sample path of $\pth{X^N(t): \, t\in \reals_+}$ such that only one jump occurs at a time and no jumps occur at time $t_c=\frac{1}{\lambda}$. Note that $x^N(\cdot)$ is a vector-valued function defined over $\reals_+$.  % -- recalling that $x^N(t)$ is a function $t$.
Clearly, with probability one, a sample path of $\pth{X^N(t): \, t\in \reals_+}$ satisfies these conditions. We focus on the coordinate process $x_1^N(\cdot)$ -- the evolution of the number of individuals that prefer the best option.
Given $x_1^N(\cdot)$, the sample path of the embedded random walk $\pth{\widehat{W}(k): \, k\in \integers_+}$, denoted by $w(k), \forall k \in \integers_+$, is also determined. Note that $w(\cdot)$ is defined over $\integers_+$.  %We couple the sample paths of $\pth{\bar{W}(k): \, k\in \integers_+}$ with $\pth{W(k): \, k\in \integers_+}$ as follows:

Let $\tau_j$ for $j=1, 2, \cdots$ be the $j^{th}$ jump time during $(t_c, \infty)$, where $t_c=\frac{1}{\lambda}$. The time $t_c$ is referred as coupling starting time.
%For $j=1, \cdots $, let $\tau_j$ be the time in the given sample path $x^N(t)$ of $\pth{X^N(t): \, t\in \reals_+}$ such that the $j^{th}$ move of the random walk $\pth{W(k): \, k\in \integers_+}$ occurs. %By the choice of the sample path, at time $\tau_k$, the CTMC changes from state $\lim_{h\downarrow 0}x_n(\tau_k-h)$ to $x_n(\tau_k)$.
For ease of notation, let
\begin{align}
\label{aux: prob state}
\lim_{h\downarrow 0}x^N(\tau_j-h) ~ \triangleq  ~ \tilde{x}(\tau_j) = \qth{\tilde{x}^N_0(\tau_j), \tilde{x}^N_1(\tau_j), \cdots, x^N_K(\tau_j)}.
\end{align}

We couple $\pth{W(k): \, k\in \integers_+}$ and $\pth{\widehat{W}(k): \, k\in \integers_+}$ as follows:
Let
\begin{align}
\label{cp: start}
\widehat{w}(0) = x_1^N(t_c).
\end{align}
\begin{itemize}
\item If the embedded random walk $\pth{W(k): \, k\in \integers_+}$ moves one position {\em down} in the sample path $x^N(t)$, then we move the standard random walk $\pth{\widehat{W}: \, k\in \integers_+}$ one position down if possible. (If the standard random walk is at zero already, it stays at zero.)
\item If the embedded random walk $\pth{W(k): \, k\in \integers_+}$ moves one position {\em up} in the sample path $x^N(t)$, then we move the standard random walk $\pth{\widehat{W}(k): \, k\in \integers_+}$ one position up. Then we flip a biased coin whose probability of showing {\em head} is a function of the state of $\tilde{x}^N(\tau_j)$. If ``HEAD'', the standard random walk stays at where it is, otherwise (``TAIL''), the standard random walk moves two positions {\em down}. In particular,
\begin{align*}
\begin{cases}
\text{HEAD}, & \text{with probability } \frac{p_1}{(p_1+p_2)\eta(\tau_j)}; \\
\text{TAIL}, & \text{with probability } 1- \frac{p_1}{(p_1+p_2)\eta(\tau_j)},
\end{cases}
\end{align*}
where
\begin{align*}
\eta(\tau_j) \triangleq \frac{\tilde{x}_0^N(\tau_j) \pth{\frac{\mu}{K} +(1-\mu)\frac{\tilde{x}_1^N(\tau_j)}{N}}p_1 + \sum_{k\ge 2} \tilde{x}_k^N(\tau_j)\frac{\tilde{x}_1^N(\tau_j)}{N}p_1}{\tilde{x}_0^N(\tau_j) \pth{\frac{\mu}{K} +(1-\mu)\frac{\tilde{x}_1^N(\tau_j)}{N}}p_1 + \sum_{k\ge 2} \tilde{x}_k^N(\tau_j)\frac{\tilde{x}_1^N (\tau_j)}{N}p_1 + \tilde{x}_1^N(\tau_j) \sum_{k\ge 2} \frac{\tilde{x}_k^N(\tau_j)}{N}p_k}.
\end{align*}
\end{itemize}
%In the meanwhile, regardless of the outcome of the flipped coin, the embedded random walk does not move.

It is easy to see that the above construction is a valid coupling.

\subsection{Proof of Lemma \ref{lm: initial wealth}}
\label{app: initial wealth}

We define a collection of $\iid$ Bernoulli random variables and conclude the proof with applying Chernoff bound. For a given $t_c$ and for each $n=1, \cdots, N$, let
\begin{align}
\label{def: bour rv}
Z_n(t_c) = \bm{1}_{\sth{\text{individual $n$ wakes up only once during time $[0, t_c]$ and $M_n(t_c)=1$}}}.
\end{align}
Since an individual wakes up whenever its Poisson clock ticks and the Poisson clocks are independent among individuals, we know that $Z_n(t_c), \forall ~ n=1, \cdots, N$ are independent.  In addition, by symmetry, $Z_n(t_c), \forall ~ n=1, \cdots, N$ are identically distributed.
Recall that $X_1^N (t_c)$ is the number of individuals whose memory states are $1$ at time $t_c$, which includes the individuals that wake up multiple times.  Thus, we have
\begin{align}
\label{lm: iid}
X_1^N (t_c) \ge \sum_{n=1}^N Z_n(t_c).
\end{align}
Next we bound $\expect{Z_n(t_c) }$.
\begin{align}
\label{lm: upper bound}
\nonumber
\expect{Z_n(t_c) } &= \prob{\text{individual $n$ wakes up only once during time $[0, t_c]$ and $M_n(t_c)=1$}}\\
\nonumber
& = \prob{\text{individual $n$ wakes up only once during $[0, t_c]$}} \\
\nonumber
& ~~~ \times \prob{\text{individual $n$ updates $M$ to 1 $\mid$ individual $n$ wakes up only once during $[0, t_c]$}}\\
& \ge  \frac{\pth{t_c \lambda }^1 \exp\{- t_c \lambda\}}{1}  \times \frac{\mu}{K} p_1.
\end{align}
When $t_c=\frac{1}{\lambda}$, we have
\begin{align}
\label{eq: initial expect lower bound}
\expect{Z_n\pth{\frac{1}{\lambda}} }  \ge  \frac{\mu p_1}{K e}.
\end{align}
In fact that the lower bound in \eqref{lm: upper bound} is maximized by the choice of $t_c=\frac{1}{\lambda}$.
For any $0< \delta <1$ we have
\begin{align*}
\prob{\sum_{n=1}^N  Z_n\pth{\frac{1}{\lambda}}  \ge (1-\delta)  \frac{\mu p_1}{K e} N}
& \ge \prob{\sum_{n=1}^N  Z_n\pth{\frac{1}{\lambda}}  \ge (1-\delta)  \expect{Z_n\pth{\frac{1}{\lambda}}} N}\\
& = 1- \prob{\sum_{n=1}^N  Z_n\pth{\frac{1}{\lambda}}  < (1-\delta)  \expect{Z_n\pth{\frac{1}{\lambda}}} N}\\
& \overset{(a)}{\ge} 1 - e^{-\frac{\delta^2}{2} \expect{Z_n\pth{\frac{1}{\lambda}}} N }\\
& \ge 1-  e^{-\frac{\mu p_1}{K e} \frac{\delta^2}{2} N },
\end{align*}
where inequality (a) follows from Chernoff bound, and the last inequality follows from \eqref{eq: initial expect lower bound}.

Therefore, for any $0< \delta <1$ we have
\begin{align*}
\prob{X_1^N \pth{\frac{1}{\lambda}} \ge (1-\delta)  \frac{\mu p_1}{K e} N}
\ge 1-  e^{-\frac{\mu p_1}{K e} \frac{\delta^2}{2} N }.
\end{align*}

\subsection{Proof of Theorem \ref{thm: eventual learn}}

With Proposition \ref{prop: gam ruin} and Lemma \ref{lm: initial wealth}, we are ready to show learnability under the learning dynamics in Algorithm \ref{alg: 1}.
Recall from \eqref{event: success} that
\begin{align*}
E^N \triangleq \sth{\lim_{t\diverge } X^N(t) = x^*} = \sth{\lim_{l\diverge } X_1^{J, N}(l) = N}&\subseteq \sth{\text{every individual learns the best option}}.
\end{align*}
%We say the social group learn the best opinion if % \nb{LS: Should we mention this ``def'' earlier??? Where?}
%
\begin{proof}
With the choice of coupling starting time $t_c=\frac{1}{\lambda}$, we have
\begin{align}
\label{eq: thm eventual learn}
\nonumber
&\prob{\text{every individual learns the best option} } \ge \prob{E^N} =\prob{ \lim_{t\diverge } X^N(t) = x^*}\\
\nonumber
&= \sth{\lim_{l\diverge } X_1^{J, N}(l) = N}\\
\nonumber
& \ge \prob{X^{N}_1 \pth{\frac{1}{\lambda}}\ge (1-\delta)  \frac{\mu p_1}{K e} N, ~  \& ~ \lim_{l\diverge } X_1^{J, N}(l) = N}\\
\nonumber
& = \prob{X^{N}_1 \pth{\frac{1}{\lambda}}\ge (1-\delta)  \frac{\mu p_1}{K e} N} \prob{\lim_{l\diverge } X_1^{J, N}(l) = N \mid X^{N}_1 \pth{\frac{1}{\lambda}}\ge (1-\delta)  \frac{\mu p_1}{K e} N}\\
& \ge (1-  e^{-\frac{\mu p_1}{K e} \frac{\delta^2}{2} N }) \prob{\lim_{l\diverge } X_1^{J, N}(l) = N \mid X^{N}_1 \pth{\frac{1}{\lambda}}\ge (1-\delta)  \frac{\mu p_1}{K e} N} ,
\end{align}
where the last inequality follows from Lemma \ref{lm: initial wealth}. In addition, we have
\begin{align*}
&\sth{\lim_{k\diverge} W(k) =N \text{ and }W(k) >0 \text{ after time $\frac{1}{\lambda}$} \mid X^{N} \pth{\frac{1}{\lambda}}\ge (1-\delta)  \frac{\mu p_1}{K e} N} \\
& \subseteq \sth{\lim_{l\diverge } X_1^{J, N}(l) = N \mid X^{N}_1 \pth{\frac{1}{\lambda}}\ge (1-\delta)  \frac{\mu p_1}{K e} N}.
\end{align*}
By construction of our coupling, with the coupling starting time $t_c=\frac{1}{\lambda}$, we have
\begin{align*}
&\sth{\lim_{k^{\prime}\diverge} \widehat{W}(k^{\prime}) =N  \mid   \widehat{W}(0) \ge (1-\delta)  \frac{\mu p_1}{K e} N}\\
& \subseteq   \sth{\lim_{k\diverge} W(k) =N \text{ and }W(k) >0 \text{ after time $\frac{1}{\lambda}$} \mid X^{N} \pth{\frac{1}{\lambda}}\ge (1-\delta)  \frac{\mu p_1}{K e} N}.
\end{align*}
Thus,
\begin{align*}
\prob{\lim_{l\diverge } X_1^{J, N}(l) = N \mid X^{N}_1 \pth{\frac{1}{\lambda}}\ge (1-\delta)  \frac{\mu p_1}{K e} N} & \ge \prob{\lim_{k^{\prime}\diverge} \widehat{W}(k^{\prime}) =N  \mid   \widehat{W}(0) \ge (1-\delta)  \frac{\mu p_1}{K e} N}\\
& \ge 1- \pth{\frac{p_1}{p_2}}^{-(1-\delta)  \frac{\mu p_1}{K e} N},
\end{align*}
where the last inequality follows from Proposition \ref{prop: gam ruin}.

Therefore,
\begin{align*}
\prob{\text{every individual learns the best option} } &\ge (1-  e^{-\frac{\mu p_1}{K e} \frac{\delta^2}{2} N })  \pth{1- \pth{\frac{p_1}{p_2}}^{-(1-\delta)  \frac{\mu p_1}{K e} N} }\\
& \ge 1- \pth{\frac{p_1}{p_2}}^{-(1-\delta)  \frac{\mu p_1}{K e} N} -  e^{-\frac{\mu p_1}{K e} \frac{\delta^2}{2} N }.
\end{align*}
The proof of Theorem \ref{thm: eventual learn} is complete.

\end{proof}

\section{Transient System Behaviors}
\label{sec: finite horizon}

We first present the proofs of Lemmas \ref{thm: fluid model}, \ref{lemma: uniqueness},and \ref{thm: convergence rate ode}. Theorem \ref{thm: transient} follows immediately from these lemmas.

\subsection{Proof of Lemma \ref{thm: fluid model}}

In fact, a stronger mode of convergence, almost surely convergence, can be shown. We focus on convergence in probability in order to get  ``large deviation'' type of bounds for finite $N$.

Recall that the initial condition of $Y(t)$ equals the scaled initial states of the scaled Markov chain $\pth{\frac{X^N(t)}{N}, t\in \reals_+}$, i.e.,
\begin{align*}
Y(0) = \qth{1, 0, \cdots, 0} = \frac{1}{N} \qth{N, 0, \cdots, 0} = \frac{1}{N} X^N(0),
\end{align*}

%
%
%
%
%
%
% \begin{figure}
%  \centering
%   \includegraphics[width=8cm]{04.pdf}
%  \caption{$\mu=0.1$, $p_1=0.8$ and $p_2=0.4$ with initial state $(0, 1)$}\label{example2}
% \end{figure}
First we need to show that the solutions to the ODE system in Lemma \ref{thm: fluid model} is unique; for this purpose, it is enough to show $F$ is Lipschitz-continuous \cite{mitzenmacher1996power}.
\begin{lemma}
\label{lip continuous}
Function $F$ defined in \eqref{den function} is $\lambda \pth{5+\sqrt{K}}$-Lipschitz continuous w.r.t. $\ell_2$ norm, i.e., %for any $x, y \in \Delta^{K}$
\begin{align*}
\norm{F(x)-F(y)} \le \lambda \pth{5+\sqrt{K}} \norm{x-y}, ~~~~ \forall ~ x, y \in \Delta^{K+1}.
\end{align*}
% \nbr{Is it possible to get ride of $\sqrt{K}$?}
\end{lemma}
We prove Lemma \ref{lip continuous} in Appendix \ref{app: lm: lip continuous}.
\begin{remark}
If $\ell_1$ norm is used, similarly, we can show  $F$ is $4\lambda $--Lipschitz continuous, i.e.,
\begin{align*}
\|F(x) - F(y)\|_1 \le 4\lambda \|x - y\|_1, ~~~~ \text{for any } x, y\in \Delta^{K+1}.
\end{align*}
\end{remark}

Before presenting the formal proof of Lemma \ref{thm: fluid model}, we provide a proof sketch first.
%Now we are ready to sketch the proof of Theorem \ref{thm: fluid model}. Formal proof can be found in Appendix \ref{app: thm: fluid model}.
\begin{proof}[Proof Sketch of Lemma \ref{thm: fluid model}]
%In this section, we prove validity of our fluid approximation, i.e., proving Theorem \ref{thm: fluid model}.
Our proof follows the same line of analysis as that in the book \cite[Chapter 5.1]{shwartz1995large}. We present the proof here for completeness.

As the drift function $F(\cdot)$ is Lipschitz continuous (by Lemma \ref{lip continuous}), the solution of the ODEs system in \eqref{eq: thm: ode} is unique and can be written as
\begin{align}
\label{eq: ode: int}
Y(t) = Y(0) + \int_{0}^t F(Y(s)) ds.
\end{align}

Our proof relies crucially on the well-known Gronwall's inequality, and the fact a crucial random process associated with $\pth{X^N(t): t\in \reals_+}$ is a martingale.
\begin{lemma}[Gronwall's inequality]
\label{Gronwall's ineq}
Let $f: [0, T] \to \reals $ be a bounded function on $[0, T]$ satisfying
\begin{align*}
f(t)\le \epsilon +\delta \cdot \int_{0}^t f(s) ds, ~~~ \text{for } t \in [0, T],
\end{align*}
where $\epsilon$ and $\delta$ are positive constants. Then we have
\begin{align*}
f(t)\le \epsilon \cdot \exp\pth{\delta t}, ~~~ \text{for } t \in [0, T].
\end{align*}
\end{lemma}
The following fact is a direct consequence of Theorem 4.13 in \cite{shwartz1995large}.
\begin{fact}[Exponential Martingale]
\label{fact: exponential martingale}
If for any bounded function $h$, it holds that
\begin{align*}
\sup_{s\in \Delta^{K}} ~ \norm{ \sum_{\ell: \ell \not={\bf 0}} q^N_{s, s+\frac{\ell}{N}} \pth{h\pth{s+\frac{\ell}{N}} - h(s)}} < \infty,
\end{align*}
then the random process $\pth{M(t): t\in \reals_+}$ defined by
\begin{align}
\label{eq: exponential martingale}
M(t) ~ \triangleq ~ \exp \pth{\iprod{\frac{X^N(t)}{N}}{\theta}  - \int_{0}^t  \sum_{\ell: \ell \not={\bf 0}} ~ q^N_{\frac{X^N(s)}{N}, ~ \frac{X^N(s)}{N} +\frac{\ell}{N}} \pth{e^{\iprod{\theta}{\frac{\ell}{N}}}-1} ds }
\end{align}
is a martingale.
\end{fact}

It is easy to see the precondition of Fact \ref{fact: exponential martingale} holds in our problem: For a given bounded function $h$, there exists a constant $C_h>0$ such that
\begin{align*}
\sup_{s\in \Delta^K} ~ \norm{h(s)} \le C_h.
\end{align*}
Thus,
\begin{align}
\label{eq: bounded generator}
\nonumber
\sup_{s\in \Delta^{K}} ~ \norm{ \sum_{\ell: \ell \not={\bf 0}} q^N_{s, s+\frac{\ell}{N}} \pth{h\pth{s+\frac{\ell}{N}} - h(s)}}
& \le  \sup_{s\in \Delta^{K}} ~  \sum_{\ell: \ell \not={\bf 0}} q^N_{s, s+\frac{\ell}{N}} \norm{h\pth{s+\frac{\ell}{N}} - h(s)} \\
& \le \sup_{s\in \Delta^{K}} ~   N \lambda 2 C_h < \infty.
\end{align}

\paragraph{Proof outline}
Recall from \eqref{def: scaled process} that
\begin{align*}
Y^N(t) = \frac{X^N(t)}{N}.
\end{align*}
Based on Fact \ref{fact: exponential martingale}, use Doob's martingale inequality and standard chernoff bound type of argument, we are able to show that, with high probabilty,
\begin{align}
\label{outline: 1}
\sup_{t\in [0, T]}
\norm{Y^N(t)- Y(t) -\int_{0}^t \pth{F\pth{Y^N(s)} - F(Y(s))} ds} \le \epsilon,
\end{align}
where $\epsilon >0$ is some small quantity. Then triangle inequality, Lemma \ref{lip continuous}, and \eqref{outline: 1} imply that, with high probabilty,
\begin{align}
\label{outline: 2}
\sup_{t\in [0, T]}
\norm{Y^N(t)- Y(t)} - \lambda \pth{5+\sqrt{K}}  \int_{0}^t \norm{Y^N(s)- Y(s)} \le  ~ \epsilon,
\end{align}
where $\lambda \pth{5+\sqrt{K}} $ is the Lipschitz constant of the drift function $F(\cdot)$. Fianlly, we apply Gronwall's inequality to the set of sample paths for which \eqref{outline: 2} holds to conclude, with high probability,
\begin{align*}
\sup_{t\in [0, T]} \norm{Y^N(t)- Y(t)}\approx ~ 0.
\end{align*}

\end{proof}

\subsubsection{Proof of Lemma \ref{thm: fluid model}}
\label{app: thm: fluid model}
In the following lemma, we require $0< \epsilon < T\lambda$, which can be easily satisfied -- observing that $\epsilon$ is typically very small.
\begin{lemma}
\label{lm: martingale bound}
Fix $T$ and $N$. For any $0 <\epsilon \le T \lambda$, we have
\begin{align}
\label{eq: int difference}
\prob{\sup_{0\le t \le T} \norm{Y^N(t)- Y(t) -\int_{0}^t \pth{F\pth{Y^N(s)} - F(Y(s))} ds} \ge \epsilon \sqrt{K+1} } \le 2 \pth{K+1} \exp \pth{-N \cdot C(\epsilon)},
\end{align}
where $C(\epsilon) =  \frac{3-e}{9T\lambda} \epsilon^2$.
\end{lemma}
\begin{proof} % [Proof of Lemma \ref{lm: martingale bound}]
The idea behind the proof is similar to the idea behind large deviations of random variables: For each direction $\theta\in \reals^{K+1}$ such that $\norm{\theta} =1$, we show that with high probability
\begin{align*}
\sup_{0\le t \le T} \iprod{Y^N(t) - Y(t) -\int_{0}^t \pth{F\pth{Y^N(s)} - F(Y(s))} ds ~}{~\theta}
\end{align*}
is small by applying the concentration of exponential martingale. Then, we use union bound to conclude \eqref{eq: int difference}.

\vskip \baselineskip

From Fact \ref{fact: exponential martingale}, we know
\begin{align*}
M(t) ~ &\triangleq ~ \exp \pth{\iprod{Y^N(t)}{\theta}  - \int_{0}^t  \sum_{\ell: \ell \not= {\bf 0}} q^N_{Y^N(s), ~ Y^N(s) +\frac{\ell}{N}} \pth{e^{\iprod{\theta}{\frac{\ell}{N}}}-1} ds } \\
& =  ~ \exp \pth{\iprod{Y^N(t)}{\theta}  - \int_{0}^t  N \cdot \sum_{\ell: \ell \not= {\bf 0}} f \pth{Y^N(s), \ell}   \pth{e^{\iprod{\theta}{\frac{\ell}{N}}}-1} ds } ~~~~~ \text{by \eqref{state re: generator} and \eqref{den function}}
\end{align*}
is a martingale. Since $Y^N(0) = \sth{1, 0, \cdots, 0}$ is deterministic, the process
\begin{align*}
\tilde{M}(t) = & ~ M(t) \times \exp \pth{-\iprod{Y^N(0)}{\theta}} \\
= & ~ \exp \pth{\iprod{Y^N(t) - Y^N(0)}{\theta}  - \int_{0}^t  N \cdot \sum_{\ell: \ell \not= {\bf 0}} f \pth{Y^N(s), \ell}   \pth{e^{\iprod{\theta}{\frac{\ell}{N}}}-1} ds }
\end{align*}
is also a martingale. In addition, by tower property of martingale, we have  for all $t$
\begin{align*}
\expect{\tilde{M}(t)} & =  \tilde{M}(0) \\
&=  \exp \pth{\iprod{Y^N(0) - Y^N(0)}{\theta}  - \int_{0}^0  N \cdot \sum_{\ell: \ell \not= {\bf 0}} f \pth{Y^N(s), \ell}   \pth{e^{\iprod{\theta}{\frac{\ell}{N}}}-1} ds } \\
&= \exp \pth{0} = 1.
\end{align*}
Thus, $\pth{\tilde{M}(t): t\in \reals_+}$ is a mean one martingale.

Now we proceed to bound the probability of
\begin{align*}
\iprod{Y^N(t) - Y(t) -\int_{0}^t \pth{F\pth{Y^N(s)} - F(Y(s))} ds ~}{~\theta} \ge \epsilon.
\end{align*}
Our plan is to rewrite the above inner product to push out the martingale $\tilde{M}(t)$. For any $\rho > 0$, we have
\begin{align}
\label{inner product 111}
\nonumber
&\iprod{Y^N(t) - Y(t) -\int_{0}^t \pth{F\pth{Y^N(s)} - F(Y(s))} ds ~}{~\rho \theta} \\
\nonumber
& = \iprod{Y^N(t) - Y(0) - \int_{0}^t F(Y(s)) ds -\int_{0}^t \pth{F\pth{Y^N(s)} - F(Y(s))} ds ~}{~\rho \theta} ~~~~~~ \text{by \eqref{eq: ode: int}}\\
\nonumber
& = \iprod{Y^N(t) - Y(0)}{~\rho \theta} - \iprod{\int_{0}^t \pth{F\pth{Y^N(s)}} ds ~}{~\rho \theta} \\
\nonumber
& = \iprod{Y^N(t) - Y(0)}{~\rho \theta} - \int_{0}^t N \cdot \sum_{\ell: \ell \not= {\bf 0}} f \pth{Y^N(s), \ell}   \pth{e^{\iprod{\theta}{\frac{\ell}{N}}}-1} ds \\
\nonumber
& \quad +  \int_{0}^t N \cdot \sum_{\ell: \ell \not= {\bf 0}} f \pth{Y^N(s), \ell}   \pth{e^{\iprod{\theta}{\frac{\ell}{N}}}-1} ds - \iprod{\int_{0}^t \pth{F\pth{Y^N(s)}} ds ~}{~\rho \theta} \\
\nonumber
& = \iprod{Y^N(t) - Y(0)}{~\rho \theta} - \int_{0}^t N \cdot \sum_{\ell: \ell \not= {\bf 0}} f \pth{Y^N(s), \ell}   \pth{e^{\iprod{\theta}{\frac{\ell}{N}}}-1} ds \\
\nonumber
& \quad +  \int_{0}^t \sum_{\ell: \ell \not= {\bf 0}} f \pth{Y^N(s), \ell} \pth{ N\pth{e^{\iprod{\rho \theta}{\frac{\ell}{N}}}-1} - \iprod{\ell}{\rho \theta} ds}\\
\nonumber
&= \iprod{Y^N(t) - Y^N(0)}{~\rho \theta} - \int_{0}^t N \cdot \sum_{\ell: \ell \not= {\bf 0}} f \pth{Y^N(s), \ell}   \pth{e^{\iprod{\theta}{\frac{\ell}{N}}}-1} ds  ~~~~~ \text{since }Y(0)= Y^N(0)\\
& \quad +  \int_{0}^t \sum_{\ell: \ell \not= {\bf 0}} f \pth{Y^N(s), \ell} \pth{ N\pth{e^{\iprod{\rho \theta}{\frac{\ell}{N}}}-1} - \iprod{\ell}{\rho \theta} ds}.
\end{align}
From Taylor's expansion, the following inequality holds: For any $y\in \reals$,
\begin{align}
\label{tylor: upper bound}
N \pth{e^{\frac{y}{N}} -1} - y \le \frac{y^2}{2N} e^{\frac{|y|}{N}}.
\end{align}
When $y \ge 0$, \eqref{tylor: upper bound} can be shown easily by the fact that $e^x = \sum_{i=0}^{\infty} \frac{x^i}{i!}$, where $x = \frac{y}{N}$; when $y<0$, it can be shown that
$N \pth{e^{\frac{y}{N}} -1} - y \le N \pth{e^{\frac{-y}{N}} -1} + y \le \frac{y^2}{2N} e^{\frac{|y|}{N}}$ using $e^x = \sum_{i=0}^{\infty} \frac{x^i}{i!}$, where $x = \frac{y}{N}$.

\vskip 0.6\baselineskip

By \eqref{tylor: upper bound}, the last term in the right hand side of \eqref{inner product 111} can be bounded as follows
\begin{align}
\label{inner product 222}
\nonumber
&\iprod{Y^N(t) - Y(t) -\int_{0}^t \pth{F\pth{Y^N(s)} - F(Y(s))} ds ~}{~\rho \theta} \\
\nonumber
& \le \iprod{Y^N(t) - Y^N(0)}{~\rho \theta} - \int_{0}^t N \cdot \sum_{\ell: \ell \not= {\bf 0}} f \pth{Y^N(s), \ell}   \pth{e^{\iprod{\theta}{\frac{\ell}{N}}}-1} ds  \\
\nonumber
& \quad +  \int_{0}^t \sum_{\ell: \ell \not= {\bf 0}} f \pth{Y^N(s), \ell}  \frac{(\iprod{\ell}{\rho \theta})^2}{2N} e^{\frac{\abth{\iprod{\ell}{\rho \theta}}}{N}} \\
&  \le  \iprod{Y^N(t) - Y^N(0)}{~\rho \theta} - \int_{0}^t N \cdot \sum_{\ell: \ell \not= {\bf 0}} f \pth{Y^N(s), \ell}   \pth{e^{\iprod{\theta}{\frac{\ell}{N}}}-1} ds + t \lambda \frac{\rho^2}{N} e^{\frac{2\rho}{N}},
\end{align}
where the last inequality follows from (1) Cauchy-Schwarz inequality, (2) $\norm{\theta} =1$,  (3) $\norm{\ell}^2 \le 2$ for all $f \pth{Y^N(s), \ell}>0$, and (4) the fact that $\sum_{\ell: \ell \not= {\bf 0}} f \pth{Y^N(s), \ell} \le \lambda$.

Using the standard Chernoff trick, for any $\rho>0$, we get
\begin{align*}
&\prob{\sup_{0\le t \le T} \iprod{Y^N(t) - Y(t) -\int_{0}^t \pth{F\pth{Y^N(s)} - F(Y(s))} ds ~}{~\theta}  ~ \ge ~ \epsilon} \\
&=\prob{\sup_{0\le t \le T} \iprod{Y^N(t) - Y(t) -\int_{0}^t \pth{F\pth{Y^N(s)} - F(Y(s))} ds ~}{~\rho \theta}  ~ \ge ~ \rho\epsilon} \\
&\le  \prob{\sup_{0\le t \le T} \iprod{Y^N(t) - Y^N(0)}{~\rho \theta} - \int_{0}^t N \cdot \sum_{\ell: \ell \not= {\bf 0}} f \pth{Y^N(s), \ell}   \pth{e^{\iprod{\theta}{\frac{\ell}{N}}}-1} ds ~ +  ~ t \lambda \frac{\rho^2}{N} e^{\frac{2\rho}{N}}   \ge  \rho\epsilon} ~~ \text{by \eqref{inner product 222}}\\
& \le  \prob{\sup_{0\le t \le T} \exp^{\pth{ \iprod{Y^N(t) - Y^N(0)}{~\rho \theta} - \int_{0}^t N \cdot \sum_{\ell: \ell \not= {\bf 0}} f \pth{Y^N(s), \ell}   \pth{e^{\iprod{\theta}{\frac{\ell}{N}}}-1} ds} } ~ \ge ~ \exp^{\pth{\rho\epsilon -  ~ T \lambda \frac{\rho^2}{N} e^{\frac{2\rho}{N}}} } }\\
& \overset{(a)}{\le} \frac{\expect{\exp^{\pth{ \iprod{Y^N(T) - Y^N(0)}{~\rho \theta} - \int_{0}^T N \cdot \sum_{\ell} f \pth{Y^N(s), \ell}   \pth{e^{\iprod{\theta}{\frac{\ell}{N}}}-1} ds}}} }{\exp \sth{\rho\epsilon -  ~ T \lambda \frac{\rho^2}{N} e^{\frac{2\rho}{N}}} } \\
& = \frac{\expect{\tilde{M}(T)} }{\exp \sth{\rho\epsilon -  ~ T \lambda \frac{\rho^2}{N} e^{\frac{2\rho}{N}}}} = \exp \sth{T \lambda \frac{\rho^2}{N} e^{\frac{2\rho}{N}} - \rho\epsilon}.
\end{align*}
where inequality (a) follows from Doob's maximal martingale inequality, and the last equality holds because of the fact that $\pth{\tilde{M}(t): t\in \reals_+}$ is a mean one martingale.
\begin{fact}[Doob's Maximal Martingale Inequality]
For a continuous-time Martingale $\pth{M(t): t\in \reals^+}$, it holds that
\begin{align*}
\prob{\sup_{t\in [0, T]} M(t) ~ > ~ c} \le \frac{\expect{M(T)}}{c}, ~~~~ \text{for } c>0.
\end{align*}
\end{fact}

Now we bound the probability error bound $\exp \sth{T \lambda \frac{\rho^2}{N} e^{\frac{2\rho}{N}} - \rho\epsilon}$.
\begin{align*}
\exp \sth{T \lambda \frac{\rho^2}{N} e^{\frac{2\rho}{N}} - \rho\epsilon} = \exp \sth{ - N \pth{ \frac{\rho}{N}\epsilon - T \lambda \frac{\rho^2}{N^2} e^{\frac{2\rho}{N}}}}
\end{align*}
Choose $\rho = \frac{N\epsilon}{3T\lambda}$. By assumption $0<\epsilon \le T\lambda$, so we have $e^{\frac{2\epsilon}{3T\lambda}} \le e$.
Thus,
\begin{align*}
\pth{\frac{N\epsilon}{3T\lambda} \frac{1}{N}\epsilon - T \lambda \frac{\pth{\frac{N\epsilon}{3T\lambda}}^2}{N^2} e^{\frac{2\pth{\frac{N\epsilon}{3T\lambda}}}{N}} }
&= \frac{\epsilon^2}{3T\lambda} -  \pth{\frac{\epsilon^2}{9T\lambda}} e^{\frac{2\epsilon}{3T\lambda}} \ge \frac{\epsilon^2}{3T\lambda} -  \pth{\frac{\epsilon^2}{9T\lambda}} e\\
& = \frac{3-e}{9T\lambda} \epsilon^2 ~ \triangleq ~ C(\epsilon).
\end{align*}
Therefore, we have
\begin{align}
\label{inner product 333}
\prob{\sup_{0\le t \le T} \iprod{Y^N(t) - Y(t) -\int_{0}^t \pth{F\pth{Y^N(s)} - F(Y(s))} ds ~}{~\theta}  ~ \ge ~ \epsilon}
\le \exp \sth{-N \cdot C(\epsilon)}.
\end{align}

\begin{fact}[Union bound]
\label{union bound}
Let $Z$ be a random vector (with values in $\reals^{K+1}$). Suppose there are numbers $a$ and $\delta$ such that, for each unit-length vector $\theta \in \reals^{K+1}$,
\begin{align*}
\prob{\iprod{Z}{\theta} \ge a} \le \delta.
\end{align*}
Then
\begin{align*}
\prob{\norm{Z} \ge a \sqrt{K+1} } \le 2(K+1) \delta.
\end{align*}
\end{fact}

By union bound (Fact \ref{union bound}), we conclude that
\begin{align*}
& \prob{\sup_{0\le t \le T} \norm{Y^N(t) - Y(t) -\int_{0}^t \pth{F\pth{Y^N(s)} - F(Y(s))} ds } ~ \ge ~ \sqrt{K+1} \cdot  \epsilon}\\
& \le 2 \pth{K+1} \exp \sth{-N \cdot C(\epsilon)}.
\end{align*}
\end{proof}

\vskip 2\baselineskip

Now we are ready to finish the proof of Lemma \ref{thm: fluid model}.

For any $t \in [0, T]$,
\begin{align*}
&\norm{Y^N(t) - Y(t) -\int_{0}^t \pth{F\pth{Y^N(s)} - F(Y(s))} ds } \\
&\ge \norm{Y^N(t) - Y(t)} -\norm{\int_{0}^t \pth{F\pth{Y^N(s)} - F(Y(s))} ds }\\
& \ge \norm{Y^N(t) - Y(t)} -\int_{0}^t  \norm{\pth{F\pth{Y^N(s)} - F(Y(s))} } ds  \\
& \ge \norm{Y^N(t) - Y(t)} - \lambda \pth{5+\sqrt{K}} \int_{0}^t \norm{ Y^N(s) - Y(s)},
\end{align*}
where the last inequality follows from Lemma \ref{lip continuous} -- the Lipschitz continuity of $F(\cdot)$.
So we have
\begin{align}
\label{fliud approx 111}
\nonumber
&\prob{\sup_{0\le t \le T}  \norm{Y^N(t) - Y(t)} - \lambda \pth{5+\sqrt{K}} \int_{0}^t \norm{Y^N(s) - Y(s)} ~ \ge ~ \sqrt{K+1} \cdot  \epsilon }\\
\nonumber
&\le  \prob{\sup_{0\le t \le T} \norm{Y^N(t) - Y(t) -\int_{0}^t \pth{F\pth{Y^N(s)} - F(Y(s))} ds } ~ \ge ~ \sqrt{K+1} \cdot  \epsilon}\\
&\le 2 \pth{K+1} \exp \sth{-N \cdot C(\epsilon)},
\end{align}
where the last inequality follows from Lemma \ref{lm: martingale bound}.
In addition, we have
\begin{align*}
& 1 - \prob{\sup_{0\le t \le T} \norm{Y^N(t) - Y(t)} - \lambda \pth{5+\sqrt{K}} \int_{0}^t \norm{F\pth{Y^N(s)} - F(Y(s))} ~ \ge ~ \sqrt{K+1} \cdot  \epsilon} \\
& = \prob{\sup_{0\le t \le T} \norm{Y^N(t) - Y(t)} - \lambda \pth{5+\sqrt{K}} \int_{0}^t \norm{F\pth{Y^N(s)} - F(Y(s))} ~ \le ~ \sqrt{K+1} \cdot  \epsilon}\\
% & \le \prob{\sup_{0\le t \le T} \norm{Y^N(t) - Y(t)} \le ~ \sqrt{K+1} \cdot  \epsilon  \cdot \exp \qth{ \lambda \pth{5+\sqrt{K}}t }  }\\
& \le \prob{\sup_{0\le t \le T} \norm{Y^N(t) - Y(t)} \le ~ \sqrt{K+1} \cdot  \epsilon  \cdot \exp \qth{ \lambda \pth{5+\sqrt{K}}T }  } ~~~~ \text{by Lemma \ref{Gronwall's ineq}}\\
& = 1 - \prob{\sup_{0\le t \le T} \norm{Y^N(t) - Y(t)} \ge ~ \sqrt{K+1} \cdot  \epsilon  \cdot \exp \qth{ \lambda \pth{5+\sqrt{K}}T }  }.
\end{align*}
Thus,
\begin{align*}
& \prob{\sup_{0\le t \le T} \norm{Y^N(t) - Y(t)} \ge ~ \sqrt{K+1} \cdot  \epsilon  \cdot \exp \qth{ \lambda \pth{5+\sqrt{K}}T }  }\\
& \le \prob{\norm{Y^N(t) - Y(t)} - \lambda \pth{5+\sqrt{K}} \int_{0}^t \norm{F\pth{Y^N(s)} - F(Y(s))} ~ \ge ~ \sqrt{K+1} \cdot  \epsilon }\\
&\le  2 \pth{K+1} \exp \sth{-N \cdot C(\epsilon)} ~~~ \text{by \eqref{fliud approx 111}}.
\end{align*}
%
%
% Therefore, we conclude
% \begin{align*}
% \prob{\sup_{0\le t \le T} \norm{Y^N(t) - Y(t)} \ge ~ \sqrt{K+1} \cdot  \epsilon  \cdot \exp \qth{ \lambda \pth{5+\sqrt{K}}T }  }\le  2 \pth{K+1} \exp \sth{-N \cdot C(\epsilon)}.
% \end{align*}
Setting $\epsilon^{\prime} \triangleq \sqrt{K+1} \cdot  \epsilon  \cdot \exp \qth{ \lambda \pth{5+\sqrt{K}}T } $, we get
\begin{align*}
\prob{\sup_{0\le t \le T} \norm{Y^N(t) - Y(t)} \ge ~\epsilon^{\prime}  }\le  2 \pth{K+1} \exp \sth{-N \cdot C(\epsilon^{\prime})},
\end{align*}
where
\begin{align*}
C(\epsilon^{\prime}) = \frac{3-e}{9T \lambda} \frac{\pth{\epsilon^{\prime}}^2}{(K+1) \exp \pth{2 \lambda \pth{5+\sqrt{K}}T}},
\end{align*}
proving Lemma \ref{thm: fluid model}.

\subsection{Proof of Lemma \ref{lemma: uniqueness}}
\label{app: lemma: uniqueness}
Since ${\bf Y}^* =\qth{Y_0^*, Y_1^*, \cdots, Y_K^*}$ is an equilibrium state, it holds that for any $Y(t) = {\bf Y}^*$,
\begin{align}
\label{eq: lemma: equi state}
{\bf 0}=\frac{\partial }{\partial t} Y(t) &= F \pth{Y(t)}.
\end{align}

From \eqref{limit 0}, we have
\begin{align}
\label{eq: y0 neg}
\dot{Y_0}(t) &= - Y_0(t) \lambda \frac{\mu}{K}\sum_{k^{\prime}=1}^K p_{k^{\prime}} - Y_0(t) \lambda \sum_{k^{\prime}=1}^K (1-\mu) p_{k^{\prime}} Y_{k^{\prime}}(t)]
~ \le ~ 0.
\end{align}
In fact, it can be shown that $Y_0(t)$ decreases monotonically from 1 to 0. To illustrate this, replacing $Y_0(t)$ by $Y_0^*$ in \eqref{eq: y0 neg} and combining with  \eqref{eq: lemma: equi state}, we have
\begin{align*}
0=\dot{Y_0^*} &= - Y_0^* \lambda \frac{\mu}{K}\sum_{k^{\prime}=1}^K p_{k^{\prime}} - Y_0^* \lambda \sum_{k^{\prime}=1}^K (1-\mu) p_{k^{\prime}} Y_{k^{\prime}}^*\\
&=-Y_0^* \lambda \pth{\frac{\mu}{K}\sum_{k^{\prime}=1}^K p_{k^{\prime}}+\sum_{k^{\prime}=1}^K (1-\mu) p_{k^{\prime}} Y_{k^{\prime}}^*} ~\le ~0.
\end{align*}
Since $\frac{\mu}{K}\sum_{k^{\prime}=1}^K p_{k^{\prime}}+\sum_{k^{\prime}=1}^K (1-\mu) p_{k^{\prime}} Y_{k^{\prime}}^* \ge \frac{\mu}{K}\sum_{k^{\prime}=1}^K p_{k^{\prime}}> 0$, it holds that
\begin{align}
\label{equ null}
Y_0^*=0.
\end{align}

By \eqref{limit 1}, \eqref{eq: lemma: equi state} and \eqref{equ null}, we have $k=1, \cdots, K$, it holds that
\begin{align*}
%\label{lm: unique: cond}
0=\dot{Y_k^*}  &= Y_k^* \lambda  \sum_{k^{\prime}=1}^K (p_k-p_{k^{\prime}})Y_{k^{\prime}}^*,
\end{align*}
which implies that
\begin{align}
\label{unique equ regular}
Y_k^* =0, ~~ \text{or} ~~ \sum_{k^{\prime}=1}^K (p_k-p_{k^{\prime}})Y_{k^{\prime}}^*, ~~~~ \text{for } k=1, \cdots, K.
\end{align}

We are able to show that
\begin{align}
\label{eq: thm: unique: critical quan}
\sum_{k^{\prime}=1}^K (p_1-p_{k^{\prime}})Y_{k^{\prime}}^* =0.
\end{align}
The equality \eqref{eq: thm: unique: critical quan} is crucial it implies $Y_k^*=0$ for $k\ge 2$. As
\begin{align*}
p_1 > p_2\ge \cdots \ge p_K, ~~ \text{and ~~} Y_0^*=0,
\end{align*}
for $k\ge 2$, we have
\begin{align*}
\sum_{k^{\prime}=1}^K (p_k-p_{k^{\prime}})Y_{k^{\prime}}^* ~ = ~p_k - \sum_{k^{\prime}=1}^K p_{k^{\prime}}Y_{k^{\prime}}^*
~ < ~ p_1 - \sum_{k^{\prime}=1}^K p_{k^{\prime}}Y_{k^{\prime}}^* ~ = ~ \sum_{k^{\prime}=1}^K (p_1-p_{k^{\prime}})Y_{k^{\prime}}^* ~ = ~0.
\end{align*}
Thus, by \eqref{unique equ regular}, we know
\begin{align}
\label{equ regular}
Y_k^*=0, ~ \forall ~ k\ge 2.
\end{align}
Therefore, from \eqref{equ null}, \eqref{equ regular} and the fact that ${\bf Y}^* \in \Delta^{K}$, we know
\begin{align*}
Y_1^*=1,
\end{align*}
proving the theorem.

\vskip \baselineskip

To finish the proof of the theorem, it remains to show \eqref{eq: thm: unique: critical quan}. By \eqref{unique equ regular}, it is enough to show
\begin{align}
\label{unique nonzero}
Y_1^* > 0.
\end{align}
To show this, let's consider the differential equation in \eqref{limit 1} for $k=1$ -- the optimal option:
\begin{align}
\label{unique lower bound dynamic}
\nonumber
\dot{Y_1}(t)&=  Y_0(t) \lambda  \frac{\mu}{K} p_1 + Y_1(t) \lambda \pth{(1-\mu)p_1Y_0(t)+ \sum_{k^{\prime}=1}^K (p_1-p_{k^{\prime}})Y_{k^{\prime}}(t)}\\
\nonumber
&\ge Y_0(t) \lambda  \frac{\mu}{K} p_1 + Y_1(t) \lambda (1-\mu)p_1Y_0(t), ~~~~~~ \text{since~ } p_1\ge p_{k^{\prime}} ~ \forall k^{\prime}\\
&\ge Y_0(t) \lambda  \frac{\mu}{K} p_1\\
\nonumber
&\ge 0.
\end{align}
That is, $Y_1(t)$ increases monotonically from 0 to $Y_1^*$.
Recall that $Y_0(t)$ decreases monotonically from 1 to $Y_0^*=0$, and $Y_0(t)$ is continuous. Thus, for any $0<\epsilon_0 \le 1$, there exists $[0, t^*]$ such that
$$Y_0(t) \ge \epsilon_0.$$
From \eqref{unique lower bound dynamic}, we have
\begin{align*}
Y_1^*\ge Y_1(t^*) &= \int_{t=0}^{t^*}\dot{Y_1}(t)\ge \int_{t=0}^{t^*} Y_0(t) \lambda  \frac{\mu}{K} p_1 \ge \epsilon_0 \lambda  \frac{\mu}{K} p_1 t^* >0,
\end{align*}
proving \eqref{unique nonzero}.

Therefore, we conclude that
\begin{align*}
{\bf Y}^* =\qth{Y_0^*, Y_1^*, \cdots, Y_k^*} = \qth{0, 1, 0, \cdots, 0}.
\end{align*}
Since ${\bf Y}^*$ is an arbitrary equilibrium state vector, uniqueness of ${\bf Y}^*$ follows trivially.

\subsection{Proof of Lemma \ref{thm: convergence rate ode}}

Next we bound the convergence rate of $Y_0$. Our first characterization may be loose. However, we can use this loose bound to more refined characterization of the convergence rate of the entire $K+1$--dimensional trajectory.

From \eqref{limit 0}, we have
\begin{align}
\label{eq: con rate lb}
\dot{Y_0}(t) &= - Y_0(t) \lambda \frac{\mu}{K}\sum_{k^{\prime}=1}^K p_{k^{\prime}} - Y_0(t) \lambda \sum_{k^{\prime}=1}^K (1-\mu) p_{k^{\prime}} Y_{k^{\prime}}(t)] ~
 \le  - Y_0(t) \lambda \frac{\mu}{K}\sum_{k^{\prime}=1}^K p_{k^{\prime}}.
\end{align}
In fact, $Y_0$ decreases exponentially fast with rate at least $\lambda \frac{\mu}{K}\sum_{k^{\prime}=1}^K p_{k^{\prime}}$.   To rigorously show this, let us consider an auxiliary ODEs system:
\begin{align}
\label{aux: 0}
\dot{y_0} = - y_0 \lambda \frac{\mu}{K}\sum_{k^{\prime}=1}^K p_{k^{\prime}},
\end{align}
with initial state $y_0 (0) = Y_0(0)=1$. It is well know that the solution to the above differential equation with the given initial condition is unique
\begin{align}
\label{eq: solution: aux}
y_0(t) = \exp \sth{- \pth{\lambda \frac{\mu}{K}\sum_{k^{\prime}=1}^K p_{k^{\prime}} } t}.
\end{align}
\begin{claim}
\label{aux: claim 0}
For all $t\ge 0$, it holds that
\begin{align}
\label{aux: dom 00}
Y_0 (t) \le y_0(t).
\end{align}
\end{claim}
This claim can be shown easily by contradiction. A proof is provided in Appendix \ref{app: claim}.
An immediate consequence of Claim \ref{aux: claim 0} and \eqref{eq: solution: aux} is
\begin{align}
\label{eq: loose bound}
Y_0(t) \le \exp \sth{- \pth{\lambda \frac{\mu}{K}\sum_{k^{\prime}=1}^K p_{k^{\prime}} } t}.
\end{align}
Although the bound in \eqref{eq: loose bound} is only for one entry of $Y$, it can help us to get a convergence rate for all the $K+1$--dimensional trajectory. In addition, the obtained bound is even tighter than that in \eqref{eq: loose bound}.

We consider two cases:
\begin{itemize}
\item[Case 1:] $\frac{\mu}{K} p_1 + (1-\mu)p_1 \ge  (p_1 - p_2)$;
\item[Case 2:] $\frac{\mu}{K} p_1 + (1-\mu)p_1 \le  (p_1 - p_2)$.
\end{itemize}
In both of these cases, we will focus the dynamics of $Y_1$. At time $\bar{t}_c$, by \eqref{eq: loose bound}, we know
\begin{align*}
Y_0(\bar{t}_c) \le c, ~~ \text{and} ~~ \sum_{k=1}^K Y_{k}(\bar{t}_c) \ge 1-c.
\end{align*}
By \eqref{limit 1} and the fact that $Y_{k}(0) =0$ for all $k\ge 1$, we know
\begin{align*}
Y_1(\bar{t}_c) ~ \ge ~ \frac{1-c}{K}.
\end{align*}

\vskip \baselineskip

\noindent {\bf Case 1: $\frac{\mu}{K} p_1 + (1-\mu)p_1 \ge  (p_1 - p_2)$. }
From \eqref{limit 0}, we have
\begin{align}
\label{rate: bound 1}
\nonumber
\dot{Y_1}(t)&=  Y_0(t) \lambda  \frac{\mu}{K} p_1 + Y_1(t) \lambda \pth{(1-\mu)p_1Y_0(t)+ \sum_{k^{\prime}=1}^K (p_1-p_{k^{\prime}})Y_{k^{\prime}}(t)}\\
\nonumber
& \ge  Y_0(t) \lambda  \frac{\mu}{K} p_1 + Y_1(t) \lambda \pth{(1-\mu)p_1Y_0(t)+ (p_1-p_2)\sum_{{k^{\prime}}=2}^K Y_{k^{\prime}}(t)}\\
\nonumber
& = Y_0(t) \lambda  \frac{\mu}{K} p_1 + Y_1(t) \lambda \pth{(1-\mu)p_1Y_0(t)+ (p_1-p_2)\lambda\pth{1 - Y_0(t) - Y_1(t)}} \\
\nonumber
& = Y_0(t) \lambda  \frac{\mu}{K} p_1 + Y_1(t)Y_0(t) \lambda (1-\mu)p_1 - (p_1-p_2) \lambda Y_0(t)Y_1(t) + (p_1-p_2)\lambda\pth{1 - Y_1(t)}Y_1(t) \\
\nonumber
& \ge Y_0(t) Y_1(t)\lambda  \frac{\mu}{K} p_1 + Y_1(t)Y_0(t) \lambda (1-\mu)p_1 - (p_1-p_2)\lambda Y_0(t)Y_1(t) + (p_1-p_2)\lambda\pth{1 - Y_1(t)}Y_1(t)\\
\nonumber
& = Y_0(t) Y_1(t)\lambda \pth{\frac{\mu}{K} p_1 + (1-\mu)p_1 - (p_1 - p_2)} + (p_1-p_2)\lambda\pth{1 - Y_1(t)}Y_1(t)\\
& \ge (p_1-p_2)\lambda \pth{1 - Y_1(t)}Y_1(t),
\end{align}
where the last inequality follows from the assumption that $\frac{\mu}{K} p_1 + (1-\mu)p_1 \ge  (p_1 - p_2)$.

Let $y$ be an auxiliary ODE equation such that
\begin{align}\label{ode: aux 1}
  \dot{y} &  = (p_1-p_2)\lambda\pth{1 - y}y,
\end{align}
with
\begin{align}
\label{eq: boundary y1}
y(\bar{t}_c) \triangleq  Y_1(\bar{t}_c) ~ \ge ~ \frac{1-c}{K}.
\end{align}
Similar to Claim \ref{aux: claim 0}, it can be shown that for all $t\in [\bar{t}_c, \infty)$,
\begin{align}
\label{aux: dom 0}
Y_1 (t) \ge y(t).
\end{align}
%\begin{proof}
%
%We prove this lemma by contradiction. Suppose this lemma is not true, i.e., \eqref{aux: dom 0} does not hold. Since both $Y_1(t)$ and $y(t)$ are continuous over $[t_0, \infty)$, and $Y_1(t_0)=y(t_0)$, when \eqref{aux: dom 0} does not hold, there exists $\tilde{t} \in [t_0, \infty) $ such that
%\begin{align*}
%Y_1(\tilde{t}) = y(\tilde{t}), ~~ \text{and} ~~ Y_1(\tilde{t} + \Delta t) > y(\tilde{t}+ \Delta t),
%\end{align*}
%for any sufficiently small $\Delta t>0$. Thus, we have
%\begin{align*}
%\dot{Y_1}(\tilde{t})=\lim_{\Delta t \downarrow 0}\frac{Y_1(\tilde{t} + \Delta t)-Y_1(\tilde{t})}{\Delta t} ~ > ~  \lim_{\Delta t \downarrow 0}\frac{y(\tilde{t}+ \Delta t)-y(\tilde{t})}{\Delta t} = \dot{y}(\tilde{t}).
%\end{align*}
%contradicting \eqref{rate: bound 1}.  The proof of the lemma is complete.
%
% \end{proof}
%
Thus, the convergence rate of $y$ provides a lower bound of the convergence rate of the original ODE system.
%\nbr{We list out the details for ease of correctness checking. Some details will be removed later.}
Note that $y$ is an autonomous and separable.
%Thus, we have
%\begin{align*}
%\frac{\partial y}{\partial t} ~ & = ~ (p_1-p_2)\pth{1 - y}y,
%\end{align*}
%which implies that
%\begin{align*}
%\pth{\frac{1}{y} + \frac{1}{1-y}} \partial y ~ & = \frac{\partial y}{\pth{1 - y}y}  ~ = ~ (p_1-p_2) \partial t .
%\end{align*}
%Take integration from $\bar{t}_c$ to $t+\bar{t}_c$, we have
%\begin{align*}
%\int_{\bar{t}_c}^{t+\bar{t}_c} \pth{\frac{1}{y} + \frac{1}{1-y}} \partial y = \int_{\bar{t}_c}^{t+\bar{t}_c} (p_1-p_2) \partial t =  (p_1-p_2)  t.
%\end{align*}
%On the other hand, we have
%\begin{align*}
%\int_{\bar{t}_c}^{t+\bar{t}_c} \pth{\frac{1}{y} + \frac{1}{1-y}} \partial y ~
%&= ~ \qth{\log y}_{y(\bar{t}_c)}^{y(t+\bar{t}_c)}  - \qth{\log (1-y)}_{y(\bar{t}_c)}^{y(t+\bar{t}_c)}\\
%& = ~ \log \frac{(1-y(\bar{t}_c)) \cdot y(t+\bar{t}_c) }{y(\bar{t}_c) \pth{1-y(t+\bar{t}_c)}}.
%\end{align*}
%Thus,
We have
\begin{align*}
y(t+\bar{t}_c) &= 1- \frac{1}{\frac{y(\bar{t}_c)}{1-y(\bar{t}_c)} \exp \sth{(p_1 - p_2)\lambda t} +1}
%& = 1- \frac{1}{\frac{y(\bar{t}_c)}{1-y(\bar{t}_c)} \exp \sth{(p_1 - p_2)t}} \\
%& = 1-     \frac{1-y(\bar{t}_c)}{y(\bar{t}_c)} \exp \sth{- (p_1 - p_2)t}\\
 \ge 1 -\frac{K -1+c}{1-c}\exp \sth{- (p_1 - p_2)\lambda t},
\end{align*}
where the last inequality follows from \eqref{eq: boundary y1}.
By \eqref{aux: dom 0}, we know that $Y_1(t)\ge y(t)$. Therefore, we conclude that in the case when $\frac{\mu}{K} p_1 + (1-\mu)p_1 \ge  (p_1 - p_2)$, $Y_1(t)$ converges to 1 exponentially fast at a rate $(p_1 - p_2)\lambda $. Since the ODEs state $Y(t)\in \Delta^{K}$, i.e., $\sum_{k=0}^{K} Y_k(t) =1$ and $Y_k(t)\ge 0, ~ \forall k$, it holds that for all non-optimal arms, $Y_k(t)$ goes to 0 exponentially fast at a rate $(p_1 - p_2)\lambda $.

%\vskip \baselineskip
%
%\noindent {\bf Case 2: $\frac{\mu}{K} p_1 + (1-\mu)p_1 <  (p_1 - p_2)$.}
%
%We have
%\begin{align}
%\nonumber
%\dot{Y_1}(t)&=  Y_0(t) \lambda  \frac{\mu}{K} p_1 + Y_1(t) \lambda \pth{(1-\mu)p_1Y_0(t)+ \sum_{k^{\prime}=1}^K (p_1-p_{k^{\prime}})Y_{k^{\prime}}(t)}\\
%\nonumber
%& \ge  Y_0(t) \lambda  \frac{\mu}{K} p_1 + Y_1(t) \lambda \pth{(1-\mu)p_1Y_0(t)+ (p_1-p_2)\pth{1 - Y_0(t) - Y_1(t)}} \\
%\nonumber
%&\ge Y_0(t) \lambda  \frac{\mu}{K} p_1 + Y_1(t) \lambda \pth{(1-\mu)p_1Y_0(t)+ \pth{\frac{\mu}{K} p_1 + (1-\mu)p_1}\pth{1 - Y_0(t) - Y_1(t)}} \\
%\nonumber
%& = Y_0(t) \lambda  \frac{\mu}{K} p_1 + Y_1(t) \lambda (1-\mu)p_1Y_0(t) - Y_1(t)\pth{\frac{\mu}{K} p_1 + (1-\mu)p_1}Y_0(t) \\
%\nonumber
%& \quad +  \pth{\frac{\mu}{K} p_1 + (1-\mu)p_1}\pth{1  - Y_1(t)}\\
%\nonumber
%& \ge Y_0(t)Y_1(t) \lambda  \frac{\mu}{K} p_1 + Y_1(t) \lambda (1-\mu)p_1Y_0(t) - Y_1(t)\pth{\frac{\mu}{K} p_1 + (1-\mu)p_1}Y_0(t) \\
%\nonumber
%& \quad +  Y_1(t)\pth{\frac{\mu}{K} p_1 + (1-\mu)p_1}\pth{1  - Y_1(t)} \\
%& = Y_1(t)\pth{\frac{\mu}{K} p_1 + (1-\mu)p_1}\pth{1  - Y_1(t)}.
%\end{align}
%
%Let $z$ be an auxiliary ODE equation such that
%\begin{align}\label{ode: aux 2}
%  \dot{z} &  = z\pth{\frac{\mu}{K} p_1 + (1-\mu)p_1}\pth{1- z},
%\end{align}
%with
%\begin{align}
%z(\bar{t}_c) ~ = ~ Y_1(\bar{t}_c) ~ \ge ~ \frac{1-c}{K}.
%\end{align}
%%
%%
%%

Similar to Case 1, we are able to conclude that in Case 2, $Y_1(t)$ converges to 1 exponentially fast with a rate $\lambda \pth{\frac{\mu}{K} p_1 + (1-\mu)p_1}$.   Since the ODEs state $Y(t)\in \Delta^{K}$, i.e., $\sum_{k=0}^{K} Y_k(t) =1$ and $Y_k(t)\ge 0, ~ \forall k$, it holds that for all non-optimal arms, $Y_k(t)$ goes to 0 exponentially fast at a rate $\lambda \pth{ \frac{\mu}{K} p_1 + (1-\mu)p_1}$.

\section{Proof of Lemma \ref{lip continuous} }
\label{app: lm: lip continuous}

By  \eqref{den function} we know that
\begin{align*}
\norm{F(x)-F(y)}  = \norm{\sum_{\ell} \ell \pth{f(x, \ell)- f(y, \ell)} },
\end{align*}
and that
\begin{align}
\label{lipscontout}
\nonumber
\sum_{\ell} \ell \pth{f(x, \ell)- f(y, \ell)}  &= \lambda \sum_{k=1}^K e^k p_k \pth{ \frac{\mu}{K} (y_0-x_0) + (1-\mu)(y_0y_k-x_0x_k)} \\
&\quad + \lambda \sum_{k,k^{\prime}: k\not=k^{\prime}, \text{and}~ k\not=0} \pth{e^{k^{\prime}}-e^k} p_{k^{\prime}} \pth{y_{k^{\prime}}y_k- x_{k^{\prime}}x_k}.
\end{align}
Note that $e^{k^{\prime}}-e^k={\bf 0}$ for $k^{\prime}=k$. Thus, we have
\begin{align*}
\sum_{k,k^{\prime}: k\not=k^{\prime}, \text{and}~ k\not=0} \pth{e^{k^{\prime}}-e^k} p_{k^{\prime}} \pth{y_{k^{\prime}}y_k- x_{k^{\prime}}x_k} =\sum_{k,k^{\prime} =1}^K  \pth{e^{k^{\prime}}-e^k} p_{k^{\prime}} \pth{y_{k^{\prime}}y_k- x_{k^{\prime}}x_k},
\end{align*}
and \eqref{lipscontout} can be simplified as follows:
\begin{align*}
\sum_{\ell} \ell \pth{f(x, \ell)- f(y, \ell)}  &= \lambda \sum_{k=1}^K e^k p_k \pth{ \frac{\mu}{K} (y_0-x_0) + (1-\mu)(y_0y_k-x_0x_k)} \\
&\quad + \lambda \sum_{k,k^{\prime} =1}^K \pth{e^{k^{\prime}}-e^k} p_{k^{\prime}} \pth{y_{k^{\prime}}y_k- x_{k^{\prime}}x_k}.
\end{align*}
By triangle inequality, we have
\begin{align}
\label{lip cont}
\nonumber
&\norm{F(x)-F(y)} = \norm{\sum_{\ell} \ell \pth{f(x, \ell)- f(y, \ell)} }\\
& \le \lambda \norm{\sum_{k=1}^K e^k p_k \pth{ \frac{\mu}{K} (y_0-x_0) + (1-\mu)(y_0y_k-x_0x_k)}} +\lambda \norm{\sum_{k,k^{\prime} =1}^K\pth{e^{k^{\prime}}-e^k} p_{k^{\prime}} \pth{y_{k^{\prime}}y_k- x_{k^{\prime}}x_k}}.
%&\le \lambda \sqrt{K} \norm{\sum_{k=1}^K e^i \pth{ \frac{\mu}{K} (y_0-x_0) + (1-\mu)(y_0y_k-x_0x_k)}} +\lambda \sqrt{K}\norm{\sum_{k,k^{\prime} =1}^K \pth{e^{k^{\prime}}-e^k} \pth{y_{k^{\prime}}y_k- x_{k^{\prime}}x_k}},
\end{align}
%where the last inequality follows from Cauchy-Schwarz inequality and the fact that
%\begin{align*}
%\sqrt{\sum_{k=1}^K p_i^2}\le \sqrt{K}.
%\end{align*}
%\nbr{Can we tighten the bound to get rid of the dependency on $K$???}
We bound the two terms in the right-hand side of \eqref{lip cont} respectively. First, we notice that for any $k,k^{\prime}\in \{0, 1, \cdots, K\}$, it holds that
\begin{align}
\label{lip cont aux}
\nonumber
\pth{y_{k^{\prime}}-x_{k^{\prime}}} \pth{y_k+x_k} & = y_{k^{\prime}}y_k + y_{k^{\prime}}x_k - x_{k^{\prime}}y_k  - x_{k^{\prime}}x_k, \\
\pth{y_{k^{\prime}}+x_{k^{\prime}}} \pth{y_k-x_k} & = y_{k^{\prime}}y_k - y_{k^{\prime}}x_k + x_{k^{\prime}}y_k  - x_{k^{\prime}}x_k.
\end{align}
Thus, $y_{k^{\prime}}y_k - x_{k^{\prime}}x_k$ can be rewritten as follows.
\begin{align}
\label{rewriting}
y_{k^{\prime}}y_k - x_{k^{\prime}}x_k =\frac{1}{2}\pth{\pth{y_{k^{\prime}}-x_{k^{\prime}}} \pth{y_k+x_k}+ \pth{y_{k^{\prime}}+x_{k^{\prime}}} \pth{y_k-x_k}}.
\end{align}
Using \eqref{rewriting}, the second term in \eqref{lip cont} can be rewritten as:
\begin{align*}
&\sum_{k,k^{\prime} =1}^K \pth{e^{k^{\prime}}-e^k} p_{k^{\prime}} \pth{y_{k^{\prime}}y_k- x_{k^{\prime}}x_k}\\
&= \sum_{k^{\prime}=1}^K e^{k^{\prime}} p_{k^{\prime}}\sum_{k=1}^K\pth{y_{k^{\prime}}y_k- x_{k^{\prime}}x_k}- \sum_{k=1}^K e^k \sum_{k^{\prime}=1}^K p_{k^{\prime}}\pth{y_{k^{\prime}}y_k- x_{k^{\prime}}x_k}\\
& = \sum_{k^{\prime}=1}^K e^{k^{\prime}} p_{k^{\prime}} \sum_{k=1}^K\frac{1}{2}\pth{\pth{y_{k^{\prime}}-x_{k^{\prime}}} \pth{y_k+x_k}+ \pth{y_{k^{\prime}}+x_{k^{\prime}}} \pth{y_k-x_k}}\\
&\quad - \sum_{k=1}^K e^k \sum_{k^{\prime}=1}^K p_{k^{\prime}}   \frac{1}{2}\pth{\pth{y_{k^{\prime}}-x_{k^{\prime}}} \pth{y_k+x_k}+ \pth{y_{k^{\prime}}+x_{k^{\prime}}} \pth{y_k-x_k}}\\
&= \sum_{k^{\prime}=1}^K e^{k^{\prime}} p_{k^{\prime}} \frac{1}{2}\pth{\pth{y_{k^{\prime}}-x_{k^{\prime}}} \sum_{k=1}^K\pth{y_k+x_k}+ \pth{y_{k^{\prime}}+x_{k^{\prime}}} \sum_{k=1}^K\pth{y_k-x_k}}\\
&\quad - \sum_{k=1}^K e^k \frac{1}{2} \pth{\pth{y_k+x_k} \sum_{k^{\prime}=1}^K p_{k^{\prime}}\pth{y_{k^{\prime}}-x_{k^{\prime}}} +  \pth{y_k-x_k}\sum_{k^{\prime}=1}^K p_{k^{\prime}} \pth{y_{k^{\prime}}+x_{k^{\prime}}}}.
\end{align*}
Thus,
\begin{align}
\label{lip cont quad}
\nonumber
&\norm{\sum_{k,k^{\prime} =1}^K \pth{e^{k^{\prime}}-e^k} p_{k^{\prime}} \pth{y_{k^{\prime}}y_k- x_{k^{\prime}}x_k}}\\
\nonumber
 &\le \norm{\sum_{k^{\prime}=1}^K e^{k^{\prime}} p_{k^{\prime}} \frac{1}{2}\pth{\pth{y_{k^{\prime}}-x_{k^{\prime}}} \sum_{k=1}^K\pth{y_k+x_k}+ \pth{y_{k^{\prime}}+x_{k^{\prime}}} \sum_{k=1}^K\pth{y_k-x_k}}} \\
&+\norm{\sum_{k=1}^K e^k \frac{1}{2} \pth{\pth{y_k+x_k} \sum_{k^{\prime}=1}^K p_{k^{\prime}}\pth{y_{k^{\prime}}-x_{k^{\prime}}} +  \pth{y_k-x_k}\sum_{k^{\prime}=1}^K p_{k^{\prime}} \pth{y_{k^{\prime}}+x_{k^{\prime}}}}}.
\end{align}
We bound the right-hand side of \eqref{lip cont quad} as
\begin{align*}
&\norm{\sum_{k^{\prime}=1}^K e^{k^{\prime}} p_{k^{\prime}} \frac{1}{2}\pth{\pth{y_{k^{\prime}}-x_{k^{\prime}}} \sum_{k=1}^K\pth{y_k+x_k}+ \pth{y_{k^{\prime}}+x_{k^{\prime}}} \sum_{k=1}^K\pth{y_k-x_k}}}\\
&\le  \norm{\sum_{k^{\prime}=1}^K e^{k^{\prime}} p_{k^{\prime}} \frac{1}{2}\pth{y_{k^{\prime}}-x_{k^{\prime}}} \sum_{k=1}^K\pth{y_k+x_k}} + \norm{\sum_{k^{\prime}=1}^K e^{k^{\prime}} p_{k^{\prime}} \frac{1}{2} \pth{y_{k^{\prime}}+x_{k^{\prime}}} \sum_{k=1}^K\pth{y_k-x_k}}\\
&= \sqrt{ \sum_{k^{\prime}=1}^K \frac{p_{k^{\prime}}^2}{4}\pth{y_{k^{\prime}}-x_{k^{\prime}}}^2 \pth{\sum_{k=1}^K\pth{y_k+x_k}}^2} + \sqrt{\sum_{k^{\prime}=1}^K \frac{p_{k^{\prime}}^2}{4} \pth{y_{k^{\prime}}+x_{k^{\prime}}}^2 \pth{\sum_{k=1}^K\pth{y_k-x_k}}^2} \\
&\le \sqrt{ \sum_{k^{\prime}=1}^K \pth{y_{k^{\prime}}-x_{k^{\prime}}}^2} + \abth{\sum_{k=1}^K\pth{y_k-x_k}}\sqrt{ \frac{1}{4} \sum_{k^{\prime}=1}^K \pth{y_{k^{\prime}}+x_{k^{\prime}}}^2}\\
&\le \norm{x-y} + \abth{x_0-y_0} \le 2\norm{x-y}.
\end{align*}
Similarly, we have
\begin{align*}
&\norm{\sum_{k=1}^K e^k \frac{1}{2} \pth{\pth{y_k+x_k} \sum_{k^{\prime}=1}^K p_{k^{\prime}}\pth{y_{k^{\prime}}-x_{k^{\prime}}} +  \pth{y_k-x_k}\sum_{k^{\prime}=1}^K p_{k^{\prime}} \pth{y_{k^{\prime}}+x_{k^{\prime}}}}}\\
&\le \norm{\sum_{k=1}^K e^k \frac{1}{2}\pth{y_k+x_k} \sum_{k^{\prime}=1}^K p_{k^{\prime}}\pth{y_{k^{\prime}}-x_{k^{\prime}}}} +\norm{\sum_{k=1}^K e^k \frac{1}{2}\pth{y_k-x_k} \sum_{k^{\prime}=1}^K p_{k^{\prime}} \pth{y_{k^{\prime}}+x_{k^{\prime}}}},
\end{align*}
for which
\begin{align*}
\norm{\sum_{k=1}^K e^k \frac{1}{2}\pth{y_k-x_k} \sum_{k^{\prime}=1}^K p_{k^{\prime}} \pth{y_{k^{\prime}}+x_{k^{\prime}}}}&=\sqrt{\sum_{k=1}^K \frac{(y_k-x_k)^2}{4} \pth{\sum_{k^{\prime}=1}^K p_{k^{\prime}} \pth{y_{k^{\prime}}+x_{k^{\prime}}}}^2}\\
&= \sqrt{\sum_{k=1}^K \frac{(y_k-x_k)^2}{4}} \abth{\sum_{k^{\prime}=1}^K \pth{y_{k^{\prime}}+x_{k^{\prime}}}}\\
&\le \norm{y-x} \frac{1}{2} 2 = \norm{y-x},
\end{align*}
and
\begin{align*}
\norm{\sum_{k=1}^K e^k \frac{1}{2}\pth{y_k+x_k} \sum_{k^{\prime}=1}^K p_{k^{\prime}}\pth{y_{k^{\prime}}-x_{k^{\prime}}}}&= \sqrt{\sum_{k=1}^K \frac{(y_k+x_k)^2}{4}\pth{\sum_{k^{\prime}=1}^K p_{k^{\prime}}\pth{y_{k^{\prime}}-x_{k^{\prime}}}}^2 }\\
&=  \sqrt{\sum_{k=1}^K \frac{(y_k+x_k)^2}{4}} \abth{\sum_{k^{\prime}=1}^K p_{k^{\prime}}\pth{y_{k^{\prime}}-x_{k^{\prime}}}}\\
&\overset{(a)}{\le} 1 \sqrt{\sum_{k^{\prime}=1}^K p_{k^{\prime}}^2} \norm{y-x}\le \sqrt{K} \norm{y-x},
\end{align*}
where inequality (a) follows from the fact that $\sum_{k=1}^K \frac{(y_k+x_k)^2}{4} \le \frac{1}{4} \pth{\sum_{k=1}^{K}\pth{y_k+x_k}}^2$ and Cauchy-Schwarz inequality. Thus, \eqref{lip cont quad} can be bounded as
\begin{align}
\label{lip cont quad con}
\norm{\sum_{k,k^{\prime} =1}^K \pth{e^{k^{\prime}}-e^k} p_{k^{\prime}} \pth{y_{k^{\prime}}y_k- x_{k^{\prime}}x_k}} &\le  \pth{3+\sqrt{K}} \norm{x-y}.
\end{align}

The first term in \eqref{lip cont} can be bounded analogously. In particular,
%Since the quantity $\norm{\sum_{k=1}^K e^k p_k\pth{ \frac{\mu}{K} (y_0-x_0) + (1-\mu)(y_0y_k-x_0x_k)}}$ is symmetric in vectors $x$ and $y$, without loss of generality, we assume $y_0\ge x_0$.
by triangle inequality, we have
\begin{align}
\label{lip zero}
\nonumber
\norm{\sum_{k=1}^K e^k p_k \pth{ \frac{\mu}{K} (y_0-x_0) + (1-\mu)(y_0y_k-x_0x_k)}} &\le \frac{\mu}{K} \norm{\sum_{k=1}^K e^k p_k (y_0-x_0)}\\
&\quad  + (1-\mu)\norm{\sum_{k=1}^K e^k p_k\pth{y_0y_k-x_0x_k}}.
\end{align}
%
%\begin{align*}
%\norm{\sum_{k=1}^K e^k \pth{ \frac{\mu}{K} (y_0-x_0) + (1-\mu)(y_0y_k-x_0x_k)}}&\le \norm{\sum_{k=1}^K e^k \pth{ \frac{\mu}{K} (y_0-x_0) }} + (1-\mu)\norm{\sum_{k=1}^K e^k \pth{y_0y_k-x_0x_k}}.
%\end{align*}
We bound the two terms in the right-hand side of \eqref{lip zero}. For the first term, since $0\le p_k\le 1$ for all $k$, we have
\begin{align*}
\frac{\mu}{K}\norm{\sum_{k=1}^K e^k p_k\pth{y_0-x_0}} \le  \frac{\mu}{K} \sqrt{\sum_{k=1}^K p_k^2\pth{y_0-x_0}^2} \le  \frac{\mu}{K} \sqrt{K} \abth{y_0-x_0} \le \mu \norm{y-x}.
\end{align*}
%\begin{align*}
%\norm{\sum_{k=1}^K e^k \pth{ \frac{\mu}{K} (y_0-x_0) }}  \le k \frac{\mu}{K} (y_0-x_0) =  \mu(y_0-x_0)  \le \mu\norm{y-x}.
%\end{align*}
For the second term, by \eqref{rewriting}, we have
\begin{align*}
(1-\mu)\norm{\sum_{k=1}^K e^k p_k\pth{y_0y_k-x_0x_k}} &=(1-\mu) \norm{\sum_{k=1}^K e^k p_k\pth{\frac{1}{2}\pth{\pth{y_0-x_0} \pth{y_k+x_k}+ \pth{y_0+x_0} \pth{y_k-x_k}}}}\\
&\le (1-\mu)\frac{1}{2}\pth{y_0-x_0}  \norm{\sum_{k=1}^K e^k p_k\pth{y_k+x_k}}\\
& + (1-\mu)\frac{1}{2}\pth{y_0+x_0}\norm{\sum_{k=1}^K e^k p_k \pth{y_k-x_k}}.
\end{align*}
In addition, since $0\le p_k\le 1$ for all $k$ and $x_k+y_k\ge 0$, it holds that %\footnote{An alternative proof:
%$\norm{\sum_{k=1}^K e^k p_k\pth{y_k+x_k}} \le  \sum_{k=1}^K\norm{p_i\pth{y_k+x_k}} \le \sum_{k=1}^K\norm{y_i+x_i} = \sum_{k=1}^K (y_i+x_i)\le 2.$
%}
\begin{align*}
\norm{\sum_{k=1}^K e^k p_k\pth{y_k+x_k}} &=\sqrt{\sum_{k=1}^K p_k^2\pth{y_k+x_k}^2}  \le \sum_{k=1}^K (y_k+x_k)\le 2,
\end{align*}
and
\begin{align*}
\norm{\sum_{k=1}^K e^k p_k \pth{y_k-x_k}}&=\sqrt{\sum_{k=1}^K p_k^2\pth{y_k-x_k}^2} \le \sqrt{\sum_{k=1}^K \pth{y_k-x_k}^2} \le \norm{x-y}.
\end{align*}
Thus, we have
\begin{align}
\label{lip zero con}
\nonumber
\norm{\sum_{k=1}^K e^k \pth{ \frac{\mu}{K} (y_0-x_0) + (1-\mu)(y_0y_k-x_0x_k)}}&\le \frac{\mu}{K}\norm{\sum_{k=1}^K e^k \pth{y_0-x_0}} + (1-\mu)\norm{\sum_{k=1}^K e^k \pth{y_0y_k-x_0x_k}} \\
\nonumber
&\le \mu\norm{y-x} + (1-\mu)\pth{y_0-x_0} + (1-\mu) \norm{y-x}\\
&= 2\norm{y-x}.
\end{align}

From \eqref{lip cont quad con} and \eqref{lip zero con}, we conclude that
\begin{align*}
&\norm{F(x)-F(y)} \le \lambda \pth{5+\sqrt{K}} \norm{x-y}.
\end{align*}
That is, function $F$ is $\lambda \pth{5+\sqrt{K}}$-- Lipschitz continuous.

\section{Proof of Claim \ref{aux: claim 0}}
\label{app: claim}
We prove this claim by contradiction. Suppose this claim is not true, i.e., \eqref{aux: dom 00} does not hold. Since both $Y_0(t)$ and $y_0(t)$ are continuous over $[0, \infty)$, and $Y_0(0)=y(0)=1$, when \eqref{aux: dom 00} does not hold, there exists $\tilde{t} \in [0, \infty) $ such that
\begin{align*}
Y_0(\tilde{t}) = y_0(\tilde{t}), ~~ \text{and} ~~ Y_0(\tilde{t} + \Delta t) > y_0(\tilde{t}+ \Delta t),
\end{align*}
for any sufficiently small $\Delta t>0$. Thus, we have
\begin{align*}
\dot{Y_0}(\tilde{t})=\lim_{\Delta t \downarrow 0}\frac{Y_0(\tilde{t} + \Delta t)-Y_0(\tilde{t})}{\Delta t} ~ > ~  \lim_{\Delta t \downarrow 0}\frac{y_0(\tilde{t}+ \Delta t)-y_0(\tilde{t})}{\Delta t} = \dot{y}_0(\tilde{t}).
\end{align*}
contradicting \eqref{eq: con rate lb}.  The proof of the claim is complete.

\end{document}